\documentclass[11pt]{article}

\usepackage{format}
\usepackage{macros}

\title{
    From Pseudorandomness to Multi-Group Fairness and Back
    \blfootnote{
        This work was supported in part by the Simons Foundation Grant 733782 and  the Sloan Foundation Grant G-2020-13941.
    }
}
\author{
    Cynthia Dwork \\
    Harvard University and Microsoft Research \\
    \url{dwork@seas.harvard.edu}
\and
    Daniel Lee \\
    University of Pennsylvania \\
    \url{daniellee@alumni.upenn.edu}
\and
    Huijia Lin \\
    University of Washington \\
    \url{rachel@cs.washington.edu}
\and
    Pranay Tankala \\
    Harvard University \\
    \url{pranay_tankala@g.harvard.edu}
}
\date{\vspace{-30pt}}

\begin{document}

\maketitle
\noteswarning

\begin{abstract}
  We identify and explore connections between the recent literature on \emph{multi-group fairness} for prediction algorithms and the pseudorandomness notions of \emph{leakage-resilience} and \emph{graph regularity}. We frame our investigation using new variants of multicalibration based on statistical distance and closely related to the concept of \emph{outcome indistinguishability}. 
  Adopting this perspective leads us not only to new, more efficient algorithms for multicalibration, but also to our graph theoretic results and a proof of a novel \emph{hardcore lemma} for real-valued functions.
\end{abstract}


\section{Introduction}
\label{sec:introduction}

A central question in the field of algorithmic fairness concerns the extent to which prediction algorithms, which assign numeric ``probabilities'' to individuals in a population, systematically mistreat members of large demographic subpopulations. Although interest in \emph{group fairness} is far from new~\cite{turini2008discrimination,kamiran2009classifying}, the study of \emph{multi-group fairness}, originally conceived as a bridge between \emph{individual fairness}~\cite{dwork2012fairness} and group fairness notions and in which the subpopulations of interest are numerous and overlapping, is relatively young, initiated by the seminal works of \cite{hebertjohnson2018multicalibration,kearns2018preventing}. In the past few years, a fruitful line of research has investigated how to achieve various notions of multi-group fairness and their applications to learning e.g.,~\cite{gopalan2022omnipredictors, kim2019universal, deng2023happymap, gopalan2022lowdegree, jung2021moment, gupta2022online}. In this work, complementing the work of \cite{casacuberta2023finding, casacuberta2024complexity}, we excavate a deep connection between multi-group fairness and pseudorandomness, and exhibit a productive relationship between the key concept of multicalibration introduced in \cite{hebertjohnson2018multicalibration}, and notions of leakage simulation \cite{jetchev2014how}, graph regularity \cite{szemeredi1975regular,frieze1996regularity}, and hardcore lemmas \cite{impagliazzo1995hardcore}, drawing particular inspiration from \cite{trevisan2009regularity,chen2018complexity}.

\paragraph{A New Definition.}
Speaking informally, multicalibration requires that predictions be calibrated simultaneously on each member $c$ of a pre-specified collection ${\cal C}$ of arbitrarily intersecting population groups, the intuition being that a score of $v$ "means the same thing" independent of one's group membership(s)~\cite{kleinberg2016inherent}. We begin with a new variant of the definition of multicalibration that draws on the indistinguishability-based point of view of \cite{dwork2021outcome} and generalizes prior definitions along several axes. Consider a distribution ${\cal D}$ on individual-outcome pairs $(i, \os_i)$ for individuals $i \in {\cal X}$ and their associated real-world outcomes $\os_i$, and a collection ${\cal C}$ of functions capturing intersecting subpopulations in the fairness-based view and {\em distinguishers} in the outcome-indistinguishabilty framework.
Letting $\ps$ denote the real-world outcome distribution, where $\os_i \sim \Ber(\ps_i)$, the goal of a predictor $\tilde p$ is to provide outcome distributions $\tilde p_i$, for $i \in {\cal X}$,  in a way that cannot be distinguished from $\ps$ by the distinguishers. Unlike in previous work, our definition of multicalibration is in terms of {\em statistical distance}, requiring that for all $c \in {\cal C}$,
\begin{align*}
    (c(i), \tilde o_i, \tilde p_i)  \text{ and } 
    (c(i), o^*_i , \tilde p_i)  \text{ are $\varepsilon$-statistically close, where }(i,  o_i^*) \leftarrow {\cal D} \text{ and } \tilde o_i \sim \Ber(\pt_i). 
\end{align*}
(See Section~\ref{sec:definitions} for a formal treatment.) Our definition naturally accommodates outcomes $o^*_i \in {\cal O}$ for an arbitrary set ${\cal O}$ of possible outcomes, as well as functions $c: {\cal X} \rightarrow {\cal Y}$ with arbitrary ranges. A weaker variant of our definition corresponds to {\em multiaccuracy}~\cite{hebertjohnson2018multicalibration,kim2019multiaccuracy} and equivalently {\em no-access} outcome indistinguishability~\cite{dwork2021outcome} and the approximation notion in Theorem 1.1 of~\cite{trevisan2009regularity}, which in the group-fairness view only requires the predictor be accurate in expectation (rather than calibrated) on each group simultaneously. We also describe a stronger variant called {\em strict multicalibration}, which is closely related to the notion of ``swap'' multicalibration that was independently proposed by the concurrent work of \cite{gopalan2023characterizing}. The strong statistical distance condition in our definitions lends itself naturally to applications and gives a strikingly simple proof of the observation, due to~\cite{gopalan2023loss, gopalan2023characterizing}, that {\em omniprediction}~\cite{gopalan2022omnipredictors} can be achieved from multiaccuracy and overall calibration\footnote{An omnipredictor allows post-processing to obtain best-in-class (with respect to ${\cal C}$) loss with respect to any loss function in a rich set.}, as well as a generalization of omniprediction to new settings. These new settings include the \textit{multi-objective learning} problem studied in \cite{haghtalab2023unifying}, where the goal is to achieve strong performance not only across multiple loss functions but also across shifts in the underlying distribution over ${\cal X}$.

\paragraph{From Pseudorandomness to Fairness.}
The leakage simulation lemma~\cite{jetchev2014how} is a cryptographic result concerning when a few bits of auxiliary input regarding a secret can be "faked," and therefore pose no threat to secrecy. Translating to our setting, in the simplest case there is a single bit of "auxiliary" input and this corresponds to $\os_i\sim \Ber(\ps_i)$ or $\ot_i\sim\Ber(\pt_i)$, where again $\ps_i$ is the true distribution from which individual~$i$'s outcome is chosen and $\pt_i$ is the distribution proposed by the predictor. The lemma provides a construction for creating a simulator that outputs "fake" bits that fool any function in a family $\cal F$ of distinguishers that receive $(i,\os_i)$ or $(i,\ot_i)$, and is a strengthening of a conceptually similar result in~\cite{trevisan2009regularity}. The simulator is a simple combination of only a small number of functions in~$\cal F$. 

Upon inspection, multiaccuracy is a "moral equivalent" to leakage simulation. Armed with this observation, we leverage a lower bound on the size of leakage simulators~\cite{chen2018complexity} to obtain the first relative lower bound on the size of \emph{multiaccurate} predictors, refuting the possibility of having ``dream'' predictors that are more efficient than functions in ${\cal C}$. In other words, there is nothing analogous to a pseudorandom generator, {\it i.e.}, no small predictor $\tilde{p}$ can fool all polynomial-sized distinguishers.

Inspired by the specific leakage simulation algorithm of~\cite{chen2018complexity}, we next construct a general framework for multicalibration algorithms using no-regret learning. A particular instantiation of the framework results in a set of new algorithms with improved sample complexity in the \emph{multiclass} and \emph{low-degree} settings recently introduced by \cite{gopalan2022lowdegree}\footnote{\cite{gopalan2022lowdegree} defines a hierarchy of relaxations of multicalibration in which the $k$th level  ("degree $k$") constrains the first $k$ moments of the predictor, conditioned on subpopulations in ${\cal C}$, and demonstrates that some properties of multicalibration related to fairness and accuracy manifest as low-degree properties. }.

Specifically, when ${\cal C} \subseteq \{0, 1\}^{\cal X}$ is the collection of demographic subpopulations and there are $\ell \gg 2$ possible outcomes for each member of the population, our algorithm in Section~\ref{sec:no-regret-and-complexity} needs only $\log|{\cal C}|  + (1/\varepsilon)^{\ell}$ samples to achieve $({\cal C}, O(\varepsilon))$-multicalibration, ignoring factors of $(\ell/\varepsilon)^{O(1)}$. In contrast, the previous best upper bound of \cite{gopalan2022lowdegree} required $\log|{\cal C}| \times (\ell/\varepsilon)^{4\ell}$, again ignoring factors of $(\ell/\varepsilon)^{O(1)}$.\footnote{\cite{gopalan2022lowdegree} reports sample complexity in terms of ${\sf VC}({\cal C})$, which could be as large as $\log(|{\cal C}|)$.} In particular, in our sample complexity, the $\log |{\cal C}|$ term is additive and the dependence on $\ell$ is simply $\exp(\ell)$. Similarly, for degree-$k$ multicalibration our algorithm uses $\Theta\left((\log|{\cal C}| + k \log(\ell/\varepsilon))\varepsilon^{-4} \times \log(\ell)\right)$ samples, yielding an exponential improvement on the dependence in $\ell$ over the bound in~\cite{gopalan2022lowdegree} of $\Theta\left((\log|{\cal C}| + k \log(\ell/\varepsilon))\varepsilon^{-4} \times \ell\right)$.

\paragraph{From Fairness Back to Pseudorandomness} We find a tantalizing parallel between multicalibration and the \emph{Szemer\'{e}di regularity lemma} from extremal graph theory, which decomposes large dense graphs into parts that behave pseudorandomly \cite{szemeredi1975regular}. We show in Theorem~\ref{thm:regularity-fairness} that for an appropriate graph-based instantiation of the multi-group fairness framework, there is a tight correpondence between predictors satisfying \emph{strict} multicalibration and Szemer\'{e}di regularity partitions of the underlying graph. Our theorem can be viewed as an extension of a result of \cite{trevisan2009regularity} that established an anlogous link between \emph{multiaccuracy} and the \emph{Frieze-Kannan weak regularity lemma} \cite{frieze1996regularity}. It also builds on the work of \cite{skorski2017cryptographic}, which showed that the criteria of Szemer\'{e}di regularity can be phrased in terms of distinguishers---our work further shows that these distinguishers have the structure of tests for strict multicalibration with respect to an appropriate collection ${\cal C}$. Finally, by considering the standard notion of multicalibration, which is more demanding than multiaccuracy but less so than strict multicalibration, our analogy naturally leads us at a new notion of graph regularity situated between Frieze-Kannan regularity and Szemer\'{e}di regularity that we call  \emph{intermediate regularity}, which may be of independent interest.

The regularity lemma of Trevisan, Tulsiani, and Vadhan yields important implications in different areas, including the weak Szemerédi regularity lemma in graph theory, Impagliazzo’s Hardcore Lemma in complexity theory~\cite{impagliazzo1995hardcore}, the Dense Model Theorem in additive combinatorics~\cite{Imp09,reingold2008dense}, computational analogues of entropy in information theory~\cite{vadhan2012characterizing,vadhan2013uniform,Zhe14}, and weaker notions of zero-knowledge in cryptography~\cite{chung2015weak}.  Capitalizing on the increased strength of multicalibration over multiaccuracy, applying our multicalibration algorithm we derive a version of the hardcore lemma for bounded \emph{real-valued} functions with respect to natural notions of hardness and pseudorandomness\footnote{In concurrent work also capitalizing on the strength of multicalibration as a starting point, \cite{casacuberta2023finding, casacuberta2024complexity} obtains stronger and more general versions of the Hardcore Lemma, Dense Model Theorem, and characterizations of pseudo-average min-entropy.
}.

\paragraph{Organization} In Section~\ref{sec:preliminaries}, we formally state the problem setup for multi-group fairness, with an emphasis on the case of multiclass prediction. In Section~\ref{sec:definitions}, we present our new notions of multicalibration, relate them to prior definitions in the literature, and discuss their applications. In Section~\ref{sec:leakage-simulation-and-oi}, we discuss outcome indistinguishability and its relationship to leakage simulation, a connection that motivates several of our results. In Section~\ref{sec:no-regret-and-complexity}, we give an algorithm template that unifies prior algorithms for achieving outcome indistinguishability, and also derive our improved sample complexity upper bounds. In Section~\ref{sec:graph-regularity}, we detail the connection to graph regularity. In Section~\ref{sec:hardcore}, we prove our novel variant of the hardcore lemma.

\paragraph{General Notation} For sets $A$ and $B$, we let $B^A$ denote the set of functions $f : A \to B$. For $f \in B^A$, we let $f(a)$ and $f_a$ both denote the output of $f$ on input $a \in A$.

\section{The Multi-Group Fairness Framework}
\label{sec:preliminaries}

\paragraph{Individuals and Outcomes} Building on the framework introduced by \cite{dwork2021outcome}, we consider a pair $(i, o^*_i)$ of jointly distributed random variables, where $i$ is an \emph{individual} drawn from some fixed distribution over a finite \emph{population} ${\cal X}$ and $o^*_i$ is an \emph{outcome} of individual $i$ that belongs to a finite set ${\cal O}$ consisting of $\ell = |{\cal O}|$ possible outcomes.

\paragraph{Modeled Outcomes} A \emph{predictor} associates a probability distribution over possible outcomes to each member of the population. In other words, a predictor is a function $\tilde{p} : {\cal X} \to \Delta {\cal O}$, where \[\Delta {\cal O} = \left\{f \in [0, 1]^{\cal O} : \sum_{o \in {\cal O}} f(o) = 1\right\}.\] Let $\tilde{p}_j$ denote the output $\tilde{p}(j)$ of the function $\tilde{p}$ on input $j \in {\cal X}$, and let $\tilde{o}_i \in {\cal O}$ be a random variable whose conditional distribution given $i$ is specified by $\tilde{p}_i$. In other words, $\Pr[\tilde{o}_i = o \mid i] = \tilde{p}_{i}(o)$ for each possible outcome $o \in {\cal O}$. We call $\tilde{o}_i$ the \emph{modeled outcome} of individual $i$.

\paragraph{Binary Outcomes} We say that outcomes are \emph{binary} if ${\cal O} = \{0, 1\}$. In this case, we can naturally identify $\Delta {\cal O}$ with the unit interval $[0, 1]$ by mapping the distribution $\tilde{p}_j \in \Delta {\cal O}$ to the probability $\tilde{p}_j(1) \in [0, 1]$ that $\tilde{p}_j$ assigns to a positive outcome. With this convention, the conditional distribution of $\tilde{o}_i$ given $i$ is $\Ber(\tilde{p}_i)$, which is how $\tilde{o}_i$ was originally defined in \cite{dwork2021outcome}.

\paragraph{Demographic Subpopulations}
Multi-group fairness examines the ways that a predictor $\tilde{p}$ might mistreat members of large, possibly overlapping \emph{subpopulations} $S \subseteq {\cal X}$ in a prespecified collection ${\cal C}$. Each such subpopulation has an associated indicator function ${\bf 1}_S : {\cal X} \to \{0, 1\}$, and, following \cite{gopalan2022omnipredictors}, it will be notationally convenient for us to represent ${\cal C}$ directly as a collection of such functions. Concretely, we allow ${\cal C}$ to be any collection of functions $c : {\cal X} \to {\cal Y}$ for some set ${\cal Y}$. For consistency, we will write the output $c(j)$ of the function $c$ on input $j \in {\cal X}$ as $c_j$.

\paragraph{Discretization} 
We will sometimes round predictions to the nearest point in a finite set ${\cal G} \subseteq \Delta{\cal O}$, which we assume to be an $\eta$-covering of $\Delta {\cal O}$ with respect to the \emph{statistical distance metric} $\delta(f, g) = \frac{1}{2}\sum_{o \in {\cal O}}|f(o) - g(o)|$, meaning that for all $f \in \Delta {\cal O}$, there exists $g \in {\cal G}$ such that $\delta(f, g) \le \eta$. Formally, we say that $\hat p : {\cal X} \to {\cal G}$ is the \emph{discretization of $\tilde{p}$ to ${\cal G}$} if \[\hat{p}_j = \argmin_{g \in {\cal G}} \delta(\tilde{p}_j, g)\] for each $j \in {\cal X}$, breaking ties arbitrarily. We define $\hat{o}_i \in {\cal O}$ to be the modeled outcome of an individual $i$ with respect to the predictor $\hat{p}$, meaning that $\Pr[\hat{o}_i = o \mid i] = \hat{p}_{i}(o)$ for each $o \in {\cal O}$.

The size of ${\cal G}$ will also play an important role in deriving our improved complexity upper bounds in Section~\ref{sec:complexity}. We show in Section~\ref{sec:definitions} that by taking ${\cal G}$ to be the intersection of $\Delta{\cal O}$ with the grid $\left\{0, \frac{1}{m}, \frac{2}{m}, \ldots, \frac{m-1}{m}, 1\right\}^\ell$ for an appropriate integer $m$, we achieve $|{\cal G}| < (3/\eta)^{\ell-1}$. In this case, discretization to ${\cal G}$ amounts to coordinate-wise rounding.

Finally, we remark that discretization merges a predictor's level sets, and this process may destroy any special structure these level sets possess. This observation will become important in our discussion of graph regularity in Section~\ref{sec:graph-regularity}.

\section{Multicalibration via Statistical Closeness}
\label{sec:definitions}

\emph{Multiaccuracy} and \emph{multicalibration} are two essential multi-group fairness notions introduced by \cite{hebertjohnson2018multicalibration}. In this section, we state new, natural variants of these definitions in terms of \emph{statistical distance}, which, for random variables $X$ and $Y$ with finite support, is measured by the formula \[\delta(X, Y) = \max_A \big|\Pr[X \in A] - \Pr[Y \in A]\big|.\] The maximum is taken over $A \subseteq {\rm supp}(X) \cup {\rm supp}(Y)$, and we write $X \approx_\varepsilon Y$ if $\delta(X, Y) \le \varepsilon$.

\begin{definition}
\label{def:multiaccuracy}
A predictor $\tilde{p}$ is \emph{$({\cal C}, \varepsilon)$-multiaccurate} if \((c_i, \tilde{o}_i) \approx_\varepsilon (c_i, o^*_i)\) for all $c \in {\cal C}$.
\end{definition}

\begin{definition}
\label{def:multicalibration}
A predictor $\tilde{p}$ is \emph{$({\cal C}, \varepsilon)$-multicalibrated} if \((c_i, \tilde{o}_i, \tilde{p}_i) \approx_\varepsilon (c_i, o^*_i, \tilde{p}_i)\) for all $c \in {\cal C}$.
\end{definition}

A notable way in which Definitions~\ref{def:multiaccuracy} and \ref{def:multicalibration} differ from previous definitions in the multi-group fairness literature is that in the case of ${\cal C} \subseteq \{0, 1\}^{\cal X}$, our definitions concern the behavior of $\tilde{p}$ on \emph{both} the $0$-level set and $1$-level set of each $c \in {\cal C}$, as opposed to merely the $1$-level sets. We will demonstrate shortly that this distinction is unimportant if ${\cal C} \subseteq \{0, 1\}^{\cal X}$ is closed under complement, meaning that ${\bf 1}_S \in {\cal C}$ if and only if ${\bf 1}_{{\cal X} \setminus S} \in {\cal C}$.

In our discussion of graph regularity in Section~\ref{sec:graph-regularity}, we will also need a stronger variant of the multicalibration definition, which we call \emph{strict multicalibration}, that has appeared only implicitly in prior works on algorithmic fairness. To state the definition succinctly, we introduce the shorthand \[\delta(X, Y \mid Z) = \max_{A} \, \big|\Pr[X \in A \mid Z] - \Pr[Y \in A \mid Z]\big|,\] which is a function of a random variable $Z$ distributed jointly with $X$ and with $Y$. Specifically, if the value of $Z$ is $z$, then the value of $\delta(X, Y \mid Z)$ is \[\delta(X, Y \mid Z = z) = \max_{A} \, \big|\Pr[X \in A \mid Z = z] - \Pr[Y \in A \mid Z = z]\big|.\]

\begin{definition}
\label{def:strict-multicalibration}
A predictor $\tilde{p}$ is \emph{strictly $({\cal C}, \varepsilon)$-multicalibrated} if \[\E\left[\max_{c \in {\cal C}} \, \delta \big( (c_i, \tilde{o}_i), (c_i, o^*_i) \mid \tilde{p}_i \big)\right] \le \varepsilon.\]
\end{definition}

Intuitively, strict multicalibration asks that a predictor be multiaccurate on most of its level sets. As we will see shortly, Definition~\ref{def:multicalibration}, which is closest to the original definition of multicalibration and suffices for some applications, does \emph{not} require multiaccuracy on even a single level set.

In Sections \ref{sec:relationships-among} and \ref{sec:relationships-prior}, respectively, we will explain the relationships of these definitions to each other and to existing notions of multi-accuracy and multi-calibration that appear in the algorithmic fairness literature. In Section \ref{sec:omniprediction}, we demonstrate the usefulness of our new definitions by (re)proving some recent interesting results on omniprediction by~\cite{gopalan2022omnipredictors,gopalan2023loss} and deriving novel extensions.

\subsection{Relationships Among Definitions}
\label{sec:relationships-among}

First, we show that strict $({\cal C}, \varepsilon)$-multicalibration implies $({\cal C}, \varepsilon)$-multicalibration, which in turn implies $({\cal C}, \varepsilon)$-multiaccuracy:

\begin{theorem}
If $\tilde{p}$ is strictly $({\cal C}, \varepsilon)$-multicalibrated, then $\tilde{p}$ is $({\cal C}, \varepsilon)$-multicalibrated.
\end{theorem}

\begin{proof}
If $\tilde{p}$ is strictly $({\cal C}, \varepsilon)$-multicalibrated, then for all $c \in {\cal C}$, \[\delta \big( (c_i, \tilde{o}_i, \tilde{p}_i), (c_i, o^*_i, \tilde{p}_i) \big) = \E\left[\delta \big( (c_i, \tilde{o}_i), (c_i, o^*_i) \mid \tilde{p}_i \big)\right] \le \E\left[\max_{c' \in {\cal C}} \, \delta \big( (c'_i, \tilde{o}_i), (c'_i, o^*_i) \mid \tilde{p}_i \big)\right] \le \varepsilon,\] so $\tilde{p}$ is $({\cal C}, \varepsilon)$-multicalibrated, as well.
\end{proof}

\begin{theorem}
If $\tilde{p}$ is $({\cal C}, \varepsilon)$-multicalibrated, then $\tilde{p}$ is $({\cal C}, \varepsilon)$-multiaccurate.
\end{theorem}

\begin{proof}
If $\tilde{p}$ is $({\cal C}, \varepsilon)$-multicalibrated, then for all $c \in {\cal C}$, \[\delta \big( (c_i, \tilde{o}_i), (c_i, o^*_i) \big) \le \delta \big( (c_i, \tilde{o}_i, \tilde{p}_i), (c_i, o^*_i, \tilde{p}_i) \big) \le \varepsilon,\] so $\tilde{p}$ is $({\cal C}, \varepsilon)$-multiaccurate, as well.
\end{proof}

Next, we prove that the following theorem, which shows that the converse implications do not hold.

\begin{theorem}
\label{thm:multicalibration-separation}
For all $\varepsilon > 0$, there exist finite sets ${\cal X}$ and ${\cal C} \subseteq \{0, 1\}^{\cal X}$, a pair of joint random variables $(i, o^*_i) \in {\cal X} \times \{0, 1\}$, and predictors $\tilde{p}, \tilde{p}' : {\cal X} \to [0, 1]$ such that:
\begin{itemize}
    \item[(a)] $\tilde{p}$ is $({\cal C}, 0)$-multiaccurate but not $({\cal C}, 1/3)$-multicalibrated.
    \item[(b)] $\tilde{p}'$ is $({\cal C}, \varepsilon)$-multicalibrated but not strictly $({\cal C}, 1/4)$-multicalibrated.
\end{itemize}
\end{theorem}

\begin{proof}
\begin{itemize}
    \item[(a)] Let ${\cal X} = \{0, 1\}$ and ${\cal C} = \{{\bf 1}_{\cal X}\}$. Consider an individual $i$ drawn uniformly from ${\cal X}$ whose outcome $o^*_i$ is conditionally distributed as $\Ber(1/2)$ given $i$. Then the predictor $\tilde{p}_j = j$ is $({\cal C}, 0)$-multiaccurate but not $({\cal C}, 1/2-\alpha)$-multicalibrated for any $\alpha > 0$.
    \item[(b)] Let ${\cal X} = [m] \times [m]$ for some positive integer $m$, and let ${\cal C} = \{c_k : k \in [m]\}$ where \[c_k(j) = {\bf 1}[j_1 = k \text{ and } j_2 \le k]\] for each member $j = (j_1, j_2)$ of the population. Consider an individual $i = (i_1, i_2)$ drawn uniformly from ${\cal X}$ whose outcome is $o^*_i = {\bf 1}[i_1 \ge i_2]$, and let $\tilde{p}_j = j_1 / m$. The range of $\tilde{p}$ is $\{1/m, 2/m, \ldots 1\}$. A simple calculation shows that for each function $c_k \in {\cal C}$ and each value $v$ in the range of $\tilde{p}$, we have that $\Pr[\tilde{p}_i = v] = 1/m$ and \[\delta \big( (c_{ki}, \tilde{o}_i), (c_{ki}, o^*_i) \mid \tilde{p}_i = v \big) = \begin{cases} 2v(1-v) & \text{if } v = k/m \\ 0 &\text{otherwise} \end{cases}\] (recall from Section \ref{sec:preliminaries} that $c_{ki} = c_k(i)$). Therefore, as $m \to \infty$, \[\max_{c_k \in {\cal C}} \, \delta \big( (c_{ki}, \tilde{o}_i, \tilde{p}_i), (c_{ki}, o^*_i, \tilde{p}_i) \big) = \max_{k \in [m]} \, \frac{2(k/m)(1-k/m)}{m} \to 0\] but \[\E\left[\max_{c_k \in {\cal C}} \, \delta \big( (c_{ki}, \tilde{o}_i), (c_{ki}, o^*_i) \mid \tilde{p}_i \big)\right] = \sum_{k=1}^m \frac{2(k/m)(1-k/m)}{m} \to \frac{1}{3},\] so $\tilde{p}$ is $({\cal C}, \varepsilon)$-multicalibrated but not strictly $({\cal C}, 1/3-\alpha)$-multicalibrated for any $\alpha > 0$.
\end{itemize}
\end{proof}

The separation of multicalibration and strict multicalibration comes with an important caveat: any multicalibrated predictor $\tilde{p}$ can be \emph{discretized} to achieve strict multicalibration with respect to the same collection ${\cal C}$ but a significantly worse parameter $\varepsilon$. To state this result, recall that $\hat p_i$ denotes the discretization of $\tilde{p}_i$ to a finite $\eta$-covering ${\cal G}$ of $\Delta{\cal O}$.

\begin{theorem}
\label{thm:discretize}
If $\tilde{p}$ is $({\cal C}, \varepsilon)$-multicalibrated, then $\hat p$ is strictly $({\cal C}, |{\cal G}|\varepsilon + \eta)$-multicalibrated.
\end{theorem}

\begin{proof}
Since ${\cal G}$ is an $\eta$-covering of $\Delta{\cal O}$, then the inequality $\delta(\hat{o}_i, \tilde{o}_i \mid i) \le \eta$ holds almost surely, so
\[\E\left[\max_{c \in {\cal C}} \, \delta \big( (c_i, \hat{o}_i), (c_i, o^*_i) \mid \hat{p}_i \big)\right] \le \E\left[\max_{c \in {\cal C}} \, \delta \big( (c_i, \tilde{o}_i), (c_i, o^*_i) \mid \hat{p}_i \big)\right] + \eta\] and the expectation on the right hand side is precisely \[\sum_{v \in {\cal G}} \max_{c \in {\cal C}} \, \delta \big( (c_i, \tilde{o}_i), (c_i, o^*_i) \mid \hat{p}_i = v \big) \Pr[\hat{p}_i = v].\] Each of $|{\cal G}|$ terms in the sum can be bounded as follows:
\begin{align*}
    \delta \big( (c_i, \tilde{o}_i), (c_i, o^*_i) \mid \hat{p}_i = v \big) \Pr[\hat{p}_i = v] &\le \delta \big( (c_i, \tilde{o}_i, \hat{p}_i), (c_i, o^*_i, \hat{p}_i) \big) \\
    &\le \delta \big( (c_i, \tilde{o}_i, \tilde{p}_i), (c_i, o^*_i, \tilde{p}_i) \big) & \text{since }\hat{p}_i\text{ is a function of }\tilde{p}_i\text{,} \\
    &\le \varepsilon &\text{since }\tilde{p}\text{ is }({\cal C}, \varepsilon)\text{-multicalibrated.}
\end{align*}
Thus, $\hat{p}$ is strictly $({\cal C}, |{\cal G}|\varepsilon + \eta)$-multicalibrated.
\end{proof}

\begin{corollary}
\label{thm:rounding}
Let $\ell = |{\cal O}|$. For sufficiently small $\varepsilon > 0$, any $\left({\cal C}, \varepsilon^\ell\right)$-multicalibrated predictor can be made strictly $({\cal C}, 4\varepsilon)$-multicalibrated by coordinate-wise rounding to a precision depending only on $\varepsilon$ and $\ell$.
\end{corollary}

This corollary follows immediately from Theorem~\ref{thm:discretize} and the following lemma, which gives us a grid ${\cal G}$ such that $|{\cal G}|\varepsilon^\ell + \eta < 4\varepsilon$ when $\eta = 3\varepsilon$ is sufficiently small.

\begin{lemma}
\label{thm:grid-size}
For all sufficiently small $\eta > 0$, the grid ${\cal G} = \Delta{\cal O} \cap \left\{0, \frac{1}{m}, \frac{2}{m}, \ldots, \frac{m-1}{m}, 1\right\}^\ell$ with $m = \lceil (\ell-1)/\eta \rceil$ is an $\eta$-covering of $\Delta{\cal O}$ of size $|{\cal G}| < (3/\eta)^{\ell-1}$.
\end{lemma}

\begin{proof}
    Given any $f \in \Delta{\cal O}$, we can find a grid point $g \in {\cal G}$ such that $f$ and $g$ differ by at most $1/m$ in all but one coordinate. Since $(\ell-1)/m \le \eta$, this means that ${\cal G}$ is an $\eta$-covering. A counting argument shows that the size of $G$ is \[|{\cal G}| = \binom{m+\ell-1}{\ell-1} \le \left(\frac{e(m+\ell-1)}{\ell-1}\right)^{\ell-1} \le \left(e\left(\frac{1}{\eta}+2\right)\right)^{\ell-1} < \left(\frac{3}{\eta}\right)^{\ell-1}\] for all $\ell \ge 2$ and all sufficiently small $\eta > 0$.
\end{proof}

The relationships among our new definitions are depicted in Figure~\ref{fig:definitions}.

\begin{figure}[tb]
    \centering
    \[\begin{tikzcd}
        {\text{strict multicalibration}} && {\text{multicalibration}} && {\text{multiaccuracy}}
        \arrow[shift left=2, from=1-1, to=1-3]
        \arrow[from=1-3, to=1-5]
        \arrow[shift left=2, dashed, from=1-3, to=1-1]
    \end{tikzcd}\]
    \caption{Implications (solid arrow: rounding not required, dashed arrow: rounding required)}
    \label{fig:definitions}
\end{figure}

\subsection{Relationships to Prior Definitions}
\label{sec:relationships-prior}

In this section, we explain the relationships of our new definitions to existing notions of multiaccuracy and multicalibration that appear in the algorithmic fairness literature. We emphasize that strict multicalibration has not been explicitly defined in prior works. Nevertheless, their algorithms actually achieve this stronger notion.

The original definitions of multiaccuracy and multicalibration from the algorithmic fairness literature roughly correspond to our notions of the same names when we restrict attention to binary outcomes and ${\cal C} \subseteq \{0, 1\}^{\cal X}$. To facilitate the comparison, we state the following two definitions:

\begin{definition}
\label{def:multiaccuracy-conditional}
Assume ${\cal O} = \{0, 1\}$ and ${\cal C} \subseteq \{0, 1\}^{\cal X}$. We say a predictor $\tilde{p}$ is \emph{conditionally $({\cal C}, \varepsilon)$-multiaccurate} if
    \begin{align*}
        &\text{for all }{\bf 1}_S \in {\cal C} \text{ such that } \Pr[i \in S] \ge \varepsilon, \\
        &\big| \Pr[o^*_i = 1 \mid i \in S] - \Pr[\tilde{o}_i = 1 \mid i \in S] \big| \le \varepsilon.
    \end{align*}
\end{definition}

\begin{definition}
\label{def:multicalibration-conditional}
Assume ${\cal O} = \{0, 1\}$ and ${\cal C} \subseteq \{0, 1\}^{\cal X}$. We say a predictor $\tilde{p}$ is \emph{conditionally $({\cal C}, \varepsilon)$-multicalibrated} if
\begin{align*}
    &\text{for all }{\bf 1}_S \in {\cal C} \text{ such that } \Pr[i \in S] \ge \varepsilon, \\
    &\text{there exists }S' \subseteq S\text{ such that } \Pr[i \in S' \mid i \in S] \ge 1 - \varepsilon \text{ and for all } v \in {\rm supp}(\tilde{p}_i \mid i \in S'), \\
    &\Big| \Pr[o^*_i = 1 \mid i \in S' \text{ and } \tilde{p}_i = v] - v \Big| \le \varepsilon.
\end{align*}
\end{definition}

These conditional versions of multiaccuracy and multicalibration closely resemble their original definitions in \cite{hebertjohnson2018multicalibration}. They capture the intuition that the predictions of a multiaccurate (resp. multicalibrated) predictor are approximately accurate in expectation (resp. calibrated) on each subpopulation under consideration. It is also possible to give a conditional version of our definition of \emph{strict} multicalibration:

\begin{definition}
\label{def:strict-multicalibration-conditional}
Assume that ${\cal O} = \{0, 1\}$ and ${\cal C} \subseteq \{0, 1\}^{\cal X}$. We say a predictor $\tilde{p}$ is \emph{conditionally and strictly $({\cal C}, \varepsilon)$-multicalibrated} if
\begin{align*}
    &\text{there exists }V \subset [0, 1] \text{ such that } \Pr[\tilde{p}_i \in V] \ge 1 - \varepsilon \text{ and for all }v \in V,\\
    &\text{for all }{\bf 1}_S \in {\cal C} \text{ such that } \Pr[i \in S \mid \tilde{p}_i = v] \ge \varepsilon, \\
    &\Big| \Pr[o^*_i = 1 \mid i \in S \text{ and } \tilde{p}_i = v] - v \Big| \le \varepsilon.
\end{align*}
\end{definition}

Comparing Definitions~\ref{def:multicalibration-conditional} and \ref{def:strict-multicalibration-conditional} gives insight into the qualitative difference between multicalibration and strict multicalibration. Specifically, strict multicalibration \emph{reverses the order of quantifiers} in the definition of multicalibration. If $\tilde{p}$ is a strictly multicalibrated predictor, then most of its $v$-level sets satisfy the fairness guarantee uniformly across all protected subpopulations. In other words, $\tilde{p}$ is ${\cal C}$-multiaccurate on its $v$-level set. In contrast, if $\tilde{p}$ is multicalibrated but not strictly so, then each level set $\tilde{p}^{-1}(v)$ may fail the test of calibration on some subpopulation, and perhaps a different one for each $v$ in the range of $\tilde{p}$. If one intends to use multicalibration as a certificate of \emph{fairness} for a prediction algorithm, then such behavior is clearly undesirable. 

The next theorem shows that when ${\cal C} \subseteq \{0, 1\}^{\cal X}$ is closed under complement, the definitions in this section are equivalent to those of the previous section up to a polynomial change in $\varepsilon$. The proof is based on a simple application of Markov's inequality that previously appeared in \cite{gopalan2022omnipredictors}.

\begin{figure}[tb]
    \centering
    \[\begin{tikzcd}[row sep=tiny]
        {\text{Multiaccuracy:}} & {\text{Definition~\ref{def:multiaccuracy}}} && {\text{Definition~\ref{def:multiaccuracy-conditional}}} \\
        {\text{Multicalibration:}} & {\text{Definition~\ref{def:multicalibration}}} && {\text{Definition~\ref{def:multicalibration-conditional}}} \\
        {\text{Strict Multicalibration:}} & {\text{Definition~\ref{def:strict-multicalibration}}} && {\text{Definition~\ref{def:strict-multicalibration-conditional}}}
        \arrow[shift left=1, from=1-2, to=1-4]
        \arrow[shift left=1, from=1-4, to=1-2]
        \arrow[shift left=1, from=2-2, to=2-4]
        \arrow[shift left=1, from=2-4, to=2-2]
        \arrow[shift left=1, from=3-2, to=3-4]
        \arrow[shift left=1, from=3-4, to=3-2]
    \end{tikzcd}\]
    \caption{Relationships to Prior Definitions}
    \label{fig:prior-definitions}
\end{figure}

\begin{theorem}
\label{thm:definition-equivalences}
Assume ${\cal O} = \{0, 1\}$ and ${\cal C} \subseteq \{0, 1\}^{\cal X}$ is closed under complement. For each arrow from Definition \textsf{A} to Definition \textsf{B} in Figure~\ref{fig:prior-definitions}, Definition \textsf{A} with parameters $({\cal C}, \varepsilon)$ implies Definition \textsf{B} with parameters $({\cal C}, \varepsilon^c)$ for sufficiently small $\varepsilon > 0$ and an absolute constant $c \in (0, 1)$.
\end{theorem}

\begin{proof}
We consider each of the six implications separately:

\vspace{0.25cm} \noindent {\bf (\ref{def:multiaccuracy} $\implies$ \ref{def:multiaccuracy-conditional}).} If $c = {\bf 1}_S \in {\cal C}$, then \[\left|\Pr[o_i^* = 1 \mid i \in S] - \Pr[\tilde{o}_i = 1 \mid i \in S]\right| = \frac{|\Pr[o^*_i = 1, i \in S] - \Pr[\tilde{o}_i = 1, i \in S]|}{\Pr[i \in S]}.\] The numerator of this fraction is at most $\delta\big((c_i, \tilde{o}_i), (c_i, o^*_i)\big)$, which, by Definition \ref{def:multiaccuracy}, is at most $\varepsilon$. If the denominator satisfies $\Pr[i \in S] \ge \sqrt{\varepsilon}$, then the value of the fraction can be at most $\sqrt{\varepsilon}$. Thus, Definition \ref{def:multiaccuracy} with parameter $\varepsilon$ implies Definition \ref{def:multiaccuracy-conditional} with parameter $\sqrt{\varepsilon}$.

\vspace{0.25cm} \noindent {\bf (\ref{def:multicalibration} $\implies$ \ref{def:multicalibration-conditional}).} For $c = {\bf 1}_S \in {\cal C}$ and $v \in \tilde{p}({\cal X})$, consider the \emph{multicalibration violation} \[\nabla_{S, v} = \Big| \Pr[o^*_i = 1 \mid i \in S \text{ and } \tilde{p}_i = v] - v \Big| = \delta\big(\tilde{o}_i, o^*_i \mid \tilde{p}_i = v, i \in S\big),\] and observe that \[\E[\nabla_{S, \tilde{p}_i}] = \delta\big((\tilde{o}_i, \tilde{p}_i), (o^*_i, \tilde{p}_i)  \mid i \in S\big) \le \frac{\delta\big((c_i, \tilde{o}_i, \tilde{p}_i), (c_i, o^*_i, \tilde{p}_i)\big)}{\Pr[i \in S]}.\] By Definition \ref{def:multicalibration}, the numerator of this fraction is at most $\varepsilon$. If the denominator satisfies $\Pr[i \in S] \ge \varepsilon^{1/3}$, it follows that $\E[\nabla_{S, \tilde{p_i}}] \le \varepsilon^{2/3}$. Let $S' = \{j \in S : \nabla_{S, \tilde{p}_j} \le \varepsilon^{1/3}\}$. By Markov's inequality, \[\Pr[i \notin S' \mid i \in S] \le \frac{\E[\nabla_{S, \tilde{p}_i}]}{\varepsilon^{1/3}} \le \varepsilon^{1/3}.\] Thus, Definition \ref{def:multicalibration} with parameter $\varepsilon$ implies Definition \ref{def:multicalibration-conditional} with parameter $\varepsilon^{1/3}$.

\vspace{0.25cm} \noindent {\bf (\ref{def:strict-multicalibration} $\implies$ \ref{def:strict-multicalibration-conditional}).} Let $V = \{v \in \tilde{p}({\cal X}) : \max_{c \in {\cal C}} \, \delta \big( (c_i, \tilde{o}_i), (c_i, o^*_i) \mid \tilde{p}_i \big) \le \varepsilon^{2/3} \}$. By Markov, \[\Pr[\tilde{p}_i \notin V] \le \frac{\E\left[\max_{c \in {\cal C}} \, \delta \big( (c_i, \tilde{o}_i), (c_i, o^*_i) \mid \tilde{p}_i \big)\right]}{\varepsilon^{2/3}},\] which is at most $\varepsilon^{1/3}$ by Definition \ref{def:strict-multicalibration}. For $c = {\bf 1}_S \in {\cal C}$ and $v \in V$, we have \[\nabla_{S, v} = \delta\big(\tilde{o}_i, o^*_i \mid \tilde{p}_i = v, i \in S\big) \le \frac{\delta \big( (c_i, \tilde{o}_i), (c_i, o^*_i) \mid \tilde{p}_i = v \big)}{\Pr[i \in S \mid \tilde{p}_i = v]}.\] The numerator of this fraction is at most $\varepsilon^{2/3}$ by construction of $V$. If the denominator satisfies $\Pr[i \in S \mid \tilde{p}_i = v] \ge \varepsilon^{1/3}$, then it follows that $\nabla_{S, v} \le \varepsilon^{1/3}$. Thus, Definition \ref{def:strict-multicalibration} with parameter $\varepsilon$ implies Definition \ref{def:strict-multicalibration-conditional} with parameter $\varepsilon^{1/3}$.

\vspace{0.25cm} \noindent {\bf (\ref{def:multiaccuracy-conditional} $\implies$ \ref{def:multiaccuracy}).} For $c = {\bf 1}_S \in {\cal C}$, either $\Pr[i \in S] \le \varepsilon$ or $\big|\Pr[\tilde{o}_i \mid i \in S] - \Pr[o^*_i \mid i \in S]\big| \le \varepsilon$ by Definition \ref{def:multiaccuracy-conditional}. Thus, their product satisfies \(\big|\Pr[i \in S, o^*_i = 1] - \Pr[i \in S, \tilde{o}_i = 1]\big| \le \varepsilon.\) Since ${\cal C}$ is closed under complement, the same inequality holds with ${\cal X} \setminus S$ in place of $S$. It follows that \[\delta\big((c_i, \tilde{o}_i), (c_i, o^*_i)\big) = \big|\Pr[i \in S, o^*_i = 1] - \Pr[i \in S, \tilde{o}_i = 1]\big| + \big|\Pr[i \notin S, o^*_i = 1] - \Pr[i \notin S, \tilde{o}_i = 1]\big|\] is at most $2\varepsilon$. Thus, Definition \ref{def:multiaccuracy-conditional} with parameter $\varepsilon$ implies Definition \ref{def:multiaccuracy} with parameter $2\varepsilon$.

\vspace{0.25cm} \noindent {\bf (\ref{def:multicalibration-conditional} $\implies$ \ref{def:multicalibration}).} For $c = {\bf 1}_S \in {\cal C}$, we want to upper bound \[\delta\big((c_i, \tilde{o}_i, \tilde{p}_i), (c_i, o^*_i, \tilde{p}_i)\big) = \delta\big((\tilde{o}_i, \tilde{p}_i), (o^*_i, \tilde{p}_i) \mid i \in S\big)\Pr[i \in S] + \delta\big((\tilde{o}_i, \tilde{p}_i), (o^*_i, \tilde{p}_i) \mid i \notin S\big)\Pr[i \notin S].\] We will bound the two terms on the right side separately. By Definition \ref{def:multicalibration-conditional}, there exists $S' \subseteq S$ such that $\Pr[i \notin S' \mid i \in S] \le \varepsilon$ and and $\nabla_{S, \tilde{p}_j} \le \varepsilon$ for all $j \in S'$. It follows that the first term satisfies \[\delta\big((\tilde{o}_i, \tilde{p}_i), (o^*_i, \tilde{p}_i) \mid i \in S\big)\Pr[i \in S] \le \Pr[i \in S \setminus S'] + \E[\nabla_{S, \tilde{p}_i} \mid i \in S'] \le 2\varepsilon.\] Since ${\cal C}$ is closed under complement, we also have \[\delta\big((\tilde{o}_i, \tilde{p}_i), (o^*_i, \tilde{p}_i) \mid i \notin S\big)\Pr[i \notin S] \le 2\varepsilon.\] Thus, Definition \ref{def:multicalibration-conditional} with parameter $\varepsilon$ implies Definition \ref{def:multicalibration} with parameter $4\varepsilon$.

\vspace{0.25cm} \noindent {\bf (\ref{def:strict-multicalibration-conditional} $\implies$ \ref{def:strict-multicalibration}).} Take $V$ as in Definition \ref{def:strict-multicalibration-conditional}. Then \[\E\left[\max_{c \in {\cal C}} \, \delta \big( (c_i, \tilde{o}_i), (c_i, o^*_i) \mid \tilde{p}_i \big)\right] \le \Pr[\tilde{p}_i \notin V] + \E\left[\max_{c \in {\cal C}} \, \delta \big( (c_i, \tilde{o}_i), (c_i, o^*_i) \mid \tilde{p}_i \big) \mid \tilde{p}_i \in V\right].\] By our choice of $V$, the first term satisfies $\Pr[\tilde{p}_i \notin V] \le \varepsilon$. For $v \in V$, it remains to upper bound \[\delta\big((c_i, \tilde{o}_i), (c_i, o^*_i) \mid \tilde{p}_i = v\big) = \Pr[i \in S \mid \tilde{p}_i = v]\nabla_{S,v} + \Pr[i \notin S \mid \tilde{p}_i = v]\nabla_{{\cal X} \setminus S, v}.\] We will bound the two terms on the right side separately. By our choice of $V$, either $\Pr[i \in S \mid \tilde{p}_i = v] \le \varepsilon$ or $\nabla_{S, v} \le \varepsilon$. Thus, their product satisfies \(\Pr[i \in S \mid \tilde{p}_i = v] \nabla_{S,v} \le \varepsilon.\) Similarly, \(\Pr[i \notin S \mid \tilde{p}_i = v] \nabla_{{\cal X} \setminus S,v}\le \varepsilon.\) Thus, Definition \ref{def:strict-multicalibration-conditional} with parameter $\varepsilon$ implies Definition \ref{def:strict-multicalibration} with parameter $3\varepsilon$.
\end{proof}

Until this point, we have focused on the case ${\cal O} = \{0, 1\}$ and ${\cal C} \subseteq \{0, 1\}^{\cal X}$. However, there are other definitions of multi-calibration in the algorithmic fairness literature that apply to more general sets ${\cal O}$ and ${\cal C}$. One particularly noteworthy extension, introduced in Gopalan \emph{et al.} \cite{gopalan2022omnipredictors}, applies to the case that ${\cal O} = \{0, 1\}$ and ${\cal C} \subseteq [0, 1]^{\cal X}$. We include a rephrased statement here:

\begin{definition}
 \label{def:covariance-multi-calibration}
Assume ${\cal O} = \{0, 1\}$ and ${\cal C} \subseteq [0, 1]^{\cal X}$. We say \emph{$\tilde{p}$ satisfies covariance-based $({\cal C}, \varepsilon)$-multi-calibration} if \[\E \big|{\rm Cov}(c_i, o^*_i \mid \tilde{p}_i)\big| \le \varepsilon\] for all $c \in {\cal C}$.
\end{definition}

Rather than measuring the statistical distance between $(c_i, \tilde{o}_i)$ and $(c_i, o^*_i)$ given $\tilde{p}_i$ as we do, this definition measures the absolute value of the covariance of $c_i$ and $o^*_i$ given $\tilde{p}_i$. We conclude this section by showing that our version of multi-calibration is at least as strong as this covariance-based version whenever both are applicable.

\begin{theorem}\label{thm:covariance-mc-equivalence}
Assume ${\cal O} = \{0, 1\}$ and ${\cal C} \subseteq [0, 1]^{\cal X}$. If $\tilde{p}$ is $({\cal C}, \varepsilon)$-multi-calibrated, then $\tilde{p}$ also satisfies covariance-based $({\cal C}, \varepsilon)$-multi-calibration.
\end{theorem}

\begin{proof}
Fix $c : {\cal X} \to [0, 1]$ and let ${\cal Y} \subseteq [0, 1]$ be the range of $c$. Since we assume ${\cal X}$ is finite, so is ${\cal Y}$. Some straightforward algebra shows that
\begin{align*}
    {\rm Cov}(c_i, o^*_i \mid \tilde{p}_i) = \sum_{y \in {\cal Y}} \left(y - \frac{1}{2}\right)\left(\begin{array}{rr}\Pr[c_i = y, o^*_i = 1 \mid \tilde{p}_i] & + \Pr[c_i = y \mid \tilde{p}_i]\Pr[\tilde{o}_i = 1 \mid \tilde{p}_i] \\ - \Pr[c_i = y, \tilde{o}_i = 1 \mid \tilde{p}_i] & - \Pr[c_i = y \mid \tilde{p}_i] \Pr[o^*_i = 1 \mid \tilde{p}_i]\end{array}\right).
\end{align*}
We will split the above expression into a sum of two parts and bound the expected absolute value of each. First, because $|y - 1/2| \le 1/2$ for each $y \in {\cal Y}$ and $\tilde{p}$ is $({\cal C}, \varepsilon)$-multi-calibrated, we have
\begin{align*}
    \E\left|\sum_{y \in {\cal Y}} \left(y - \frac{1}{2}\right)\Big(\Pr[c_i = y, o^*_i =  1 \mid \tilde{p}_i] - \Pr[c_i = y, \tilde{o}_i = 1 \mid \tilde{p}_i]\Big)\right| \le\frac{1}{2}\E[\delta \big( (c_i, o^*_i), (c_i, \tilde{o}_i) \mid \tilde{p}_i \big) ] \le \frac{\varepsilon}{2}.
\end{align*}
Using the additional fact that $\sum_{y \in {\cal Y}} \Pr[c_i = y \mid \tilde{p}_i] = 1$, we see that
\begin{align*}
    \E\left|\sum_{y \in {\cal Y}} \left(y - \frac{1}{2}\right)\Pr[c_i = y \mid \tilde{p}_i]\big( \Pr[\tilde{o}_i = 1 \mid \tilde{p}_i] - \Pr[o^* = 1 \mid \tilde{p}_i]\Big)\right| \le\frac{1}{2}\E[\delta \big( \tilde{o}_i, o^*_i \mid \hat{p}_i \big) ] \le \frac{\varepsilon}{2}.
\end{align*}
By the triangle inequality, we conclude that \[\E \big|{\rm Cov}(c_i, o^*_i \mid \tilde{p}_i)\big| \le \frac{\varepsilon}{2} + \frac{\varepsilon}{2},\] so $\tilde{p}$ satisfies covariance-based $({\cal C}, \varepsilon)$-multi-calibration.
\end{proof}

A few remarks are in order. Although our statistical distance-based definitions handle real-valued functions $c$, our algorithms for achieving them only work with discretized ranges. When $c$ has continuous range, $c_i$ can completely describe~$i$, and statistical closeness would then force $\pt$ to be essentially equal to $\ps$. At the same time, we {\em can} achieve covariance-based multicalibration for continuous functions $c$ with range $[0,1]$ by viewing $c_i$ as a probability distribution and replacing $c_i$ with a random instantiation $b_i\sim\Ber(c_i)$ (see Lemma~\ref{lem:sd-cov} in Section~\ref{sec:hardcore}). This gives rise to a weaker statistical closeness condition that is nonetheless sufficient for some applications.

\subsection{Omniprediction}
\label{sec:omniprediction}

In this section, we show how our new, statistical distance-based notions of multiaccuracy and multicalibration lend themselves naturally to applications by giving a remarkably simple extension of a state-of-the-art result in \emph{omniprediction} to the multiclass setting. We will also extend the concept of omniprediction to consider loss functions that may depend on information of individuals. We remark that while our theorems hold whether the functions $c \in {\cal C}$ are discrete or continuous, they are only operationalizable for discrete-valued functions.

A concept introduced by \cite{gopalan2022omnipredictors}, an \emph{omnipredictor} is a single predictor capable of minimizing a wide range of \emph{loss functions} $\ell \in {\cal L}$ while achieving competitive performance against a large class of \emph{hypotheses} $c \in {\cal C}$. Informally speaking, the original omniprediction theorem \cite{gopalan2022omnipredictors} showed that any predictor $\tilde{p}$ satisfying an appropriate multicalibration condition must also be an omnipredictor for all convex, Lipschitz, and bounded loss functions. However, a recent work \cite{gopalan2023loss} made significant strides by relaxing the assumptions of this theorem while strengthening its conclusion. The stronger version of the omniprediction theorem in \cite{gopalan2023loss} only assumes that $\tilde{p}$ is multiaccurate and calibrated (not multicalibrated) and establishes omniprediction even for non-convex loss functions.

In what follows, let ${\cal C}$, as usual, be a collection of functions $c : {\cal X} \to {\cal Y}$, which we now call \emph{hypothesis} functions. Also, let ${\cal L}$ be a collection of \emph{loss} functions $\ell : {\cal Y} \times {\cal O} \to [0, 1]$. Note that each loss function $\ell$ takes as input both an outcome $o \in {\cal O}$ and an \emph{action} $y \in {\cal Y}$. It outputs a real number between $0$ and $1$ measuring the cost of choosing action $y$ when the outcome is $o$. Also, consider the following notion of post-processing a prediction to minimize a loss function, which we have modified slightly from its form in \cite{gopalan2022omnipredictors, gopalan2023loss}.

\begin{definition}
    \label{def:post-processing}
    Say that ${\rm post}_\ell : \Delta{\cal O} \to {\cal Y}$ is a \emph{post-processing function} for the loss $\ell \in {\cal L}$ if \[{\rm post}_\ell(v) \in \argmin_{y \in {\cal Y}} \E_{o \sim v}[\ell(y, o)]\] for each distribution $v \in \Delta{\cal O}$. For ease of notation, we also write $v^\ell = {\rm post}_\ell(v)$.
\end{definition}

We now recall the definition of an omnipredictor.

\begin{definition}
    \label{def:omnipredictor}
    Say that $\tilde{p} : {\cal X} \to \Delta {\cal O}$ is a $({\cal L}, {\cal C}, \varepsilon)$-omnipredictor if \[\E[\ell(\tilde{p}^\ell_i, o^*_i)] \le \E[\ell(c_i, o^*_i)] + \varepsilon\] for all $\ell \in {\cal L}$ and $c \in {\cal C}$.
\end{definition}

In order to state the theorem of interest, we first emphasize an important special case of the definition of multicalibration in Section \ref{sec:definitions}:

\begin{definition}
\label{def:calibration}
A predictor $\tilde{p}$ is \emph{$\varepsilon$-calibrated} if $(\tilde{o}_i, \tilde{p}_i) \approx_\varepsilon (o^*_i, \tilde{p}_i)$.
\end{definition}

Indeed, a predictor is $\varepsilon$-calibrated if and only if it is $(\{{\bf 1}_{\cal X}\}, \varepsilon)$-multicalibrated. We these definitions in hand, we are ready to present the main proof of this section. For clarity, we will first consider the special case that $\varepsilon = 0$.

The following two lemmas will be of use. The first says that calibrated predictions, even when post-processed, incur the same loss on real outcomes as on modeled outcomes.

\begin{lemma}
    \label{thm:calibration-omniprediction}
    If $\tilde{p}$ is $0$-calibrated, $f : \Delta{\cal O} \to {\cal Y}$ is any function, and $\ell \in {\cal L}$, then \[\E\left[\ell(f(\tilde{p}_i), o^*_i)\right] = \E\left[\ell(f(\tilde{p}_i), \tilde{o}_i)\right].\]
\end{lemma}

\begin{proof}  $0$-calibration means that $(\tilde{o}_i, \tilde{p}_i)$ and $(o^*_i, \tilde{p}_i)$ have the same joint distribution.
\end{proof}

The second lemma says that $c : {\cal X} \to {\cal Y}$ incurs the same loss on real outcomes as on modeled outcomes if the predictor is $({\cal C}, 0)$-multiaccurate.

\begin{lemma}
    \label{thm:multiaccuracy-omniprediction}
    If $\tilde{p}$ is $({\cal C}, 0)$-multiaccurate, then \[\E[\ell(c_i, o_i^*)] = \E[\ell(c_i, \tilde{o}_i)]\] for all $\ell \in {\cal L}$ and $c \in {\cal C}$.
\end{lemma}

\begin{proof} $({\cal C}, 0)$-multiaccuracy means that $(c_i, o_i^*)$ and $(c_i, \tilde{o}_i)$ have the same joint distribution.
\end{proof}

We now state and prove a rephrased version of one of the main theorems of \cite{gopalan2023loss} in terms of our new language. In \cite{gopalan2023loss}, the assumption of the theorem is that $\tilde{p}$ is multiaccurate with respect to all $[0,1]$-bounded functions of the level sets of each $c \in {\cal C}$. In our statement of the theorem, these criteria are encapsulated naturally by our statistical distance-based formulation of multiaccuracy:

\begin{theorem}
    \label{thm:simple-omniprediction}
    If $\tilde{p}$ is $\varepsilon_1$-calibrated and $({\cal C}, \varepsilon_2)$-multiaccurate, then $\tilde{p}$ is an $({\cal L}, {\cal C}, \varepsilon_1 + \varepsilon_2)$-omnipredictor.
\end{theorem}

\begin{proof} First consider the case that $\varepsilon_1 = \varepsilon_2 = 0$. Since $\tilde{p}^\ell_i$ is a function of $\tilde{p}_i$, we have
    \begin{align*}
        \E[\ell(\tilde{p}^\ell_i, o^*_i)] &= \E[\ell(\tilde{p}^\ell_i, \tilde{o}_i)] &\text{by Lemma \ref{thm:calibration-omniprediction},} \\
        &\le \E[\ell(c_i, \tilde{o}_i)] &\text{by Definition \ref{def:post-processing},} \\
        &= \E[\ell(c_i, o^*_i)] &\text{by Lemma \ref{thm:multiaccuracy-omniprediction}.}
    \end{align*}
    In the general case, standard properties of statistical distance ensure that the two expectation terms in Lemma \ref{thm:calibration-omniprediction} now differ by at most $\varepsilon_1$, and the two expectation terms in Lemma \ref{thm:multiaccuracy-omniprediction} now differ by at most $\varepsilon_2$. Here, we have used the assumption that the range of each $\ell \in {\cal L}$ is bounded between $0$~and~$1$. Adding these two slack terms to the first and third lines, respectively, of the above calculation yields \[\E[\ell(\tilde{p}^\ell_i, o^*_i)] \le \E[\ell(c_i, o^*_i)] + \varepsilon_1 + \varepsilon_2,\] so $\tilde{p}$ is an $({\cal L}, {\cal C}, \varepsilon_1 + \varepsilon_2)$-omnipredictor.
\end{proof}

For the sake of comparison, we also include a version of the original omniprediction proof of \cite{gopalan2022omnipredictors} in our current language. For simplicity, we only state the case of $\varepsilon_1 = \varepsilon_2 = 0$ but remark that additional Lipschitzness assumptions on ${\cal L}$ would be required in the case of $\varepsilon_1, \varepsilon_2 > 0$.

\begin{theorem}
    Assume ${\cal O} = \{0, 1\}$ and ${\cal Y} \subseteq [0, 1]$. If $\tilde{p}$ satisfies covariance-based $({\cal C}, 0)$-multi-calibrated and each $\ell \in {\cal L}$ is convex in its first input, then $\tilde{p}$ is an $({\cal L}, {\cal C}, 0)$-omnipredictor.
\end{theorem}

\begin{proof}
    Rephrasing the proof from \cite{gopalan2022omnipredictors} yields:
    \begin{align*}
        \E[\ell(\tilde{p}^\ell_i, o^*_i)] &= \E[\ell(\tilde{p}^\ell_i, \tilde{o}_i)] &\text{by Lemma \ref{thm:calibration-omniprediction}} \\
        &\le \E[\ell(\E[c_i\mid \tilde{p}_i], \tilde{o}_i)] &\text{by Definition \ref{def:post-processing}} \\
        &= \E[\ell(\E[c_i\mid \tilde{p}_i], o^*_i)] &\text{by Lemma \ref{thm:calibration-omniprediction}} \\
        &= \E[\ell(\E[c_i\mid \tilde{p}_i, o^*_i], o^*_i)] &\text{by Definition \ref{def:covariance-multi-calibration}} \\
        &\le \E[\ell(c_i, o^*_i)] &\text{by convexity of $\ell(-, 0)$ and $\ell(-, 1)$.}
    \end{align*}
    In the second equality, we use the fact that $\E[c_i \mid \tilde{p}_i]$ is a function of $\tilde{p}_i$. In the third equality, we used the fact that  $\E[c_i \mid \tilde{p}_i] = \E[c_i \mid \tilde{p}_i, o_i^*]$ (i.e., that $c_i$ and $o_i^*$ are conditionally uncorrelated given $\tilde{p}_i$). This is by definition of covariance-based $({\cal C}, 0)$-multi-calibration.
\end{proof}

\subsubsection{Loss Functions Dependent on Information of Individuals}
So far the concept of omniprediction considers loss functions $\ell$ depending on an action $y \in {\cal Y}$ and outcome $o \in {\cal O}$. In this section, we extend the notion to consider loss functions that may depend on additional information about an individual, that is, the loss $\ell(y, o, z)$ also depend on information $z \in {\cal T}$ of an individual $i$ w.r.t.\ whom the action $y$ and outcome $o$ are associated with. Here, we can consider different relevant information $z$ about an individual captured by a class of functions $\mathcal{Z} = \{{z}: {\cal X} \rightarrow {\cal T}\}$. Now omniprediction w.r.t.\ loss functions ${\cal L}$, hypotheses ${\cal C}$, and auxiliary information functions $\mathcal{Z}$ should guarantee that the predictor $\tilde p$ has competitive performance against hypotheses in ${\cal C}$, in minimizing a large class of losses w.r.t.\ expressive information of individuals. 

To formally state the result, we first slightly rephrase the setting of omniprediction. In prior works and above, we compare the loss incurred using an action, $y_i = {\rm post}_\ell(\tilde p_i)$, derived from postprocessing the predicted outcome distribution $\tilde p_i$, with loss incurred using actions $c_i$ prescribed by a hypothesis $c \in {\cal C}$. The postprocessing function ${\rm post}$ (Definition~\ref{def:post-processing}) ensures that $y_i$ is Bayes optimal w.r.t.\ the composed function $\ell'(p, o) = \ell({\rm post}_\ell(p), o)$, in the sense that the expectation of $\ell'$ on input $o$ drawn from a distribution $p^*$ is minimized when 
when the true probability distribution $p^*$ is given as the input --- namely, $p^* = \argmin_{p \in \Delta{\cal O}} \E_{o \sim p^*}[\ell'(p, o)]$. We say that such loss functions are {\em Bayes optimal}. Then we can equivalently state previous results w.r.t.\ a class of Bayes optimal loss functions ${\cal L'}$  and a hypothesis class ${\cal C}$ mapping individuals in ${\cal X}$ to outcome distributions $\Delta{\cal O}$ (instead of actions). Next, we describe our extension formally in this setting. 

\begin{definition}\label{def:Bayes-optimal-loss}
     Let ${\cal L}$ be a collection of \emph{loss} functions $\ell : {\cal Y} \times { \cal O} \times {\cal T} \to [0, 1]$. We say that $\cal L$ is {\em $\varepsilon$-Bayes optimal} if every $\ell \in {\cal L}$ satisfies that 
     \begin{align*}
         \forall z \in {\cal T}, \ p^*, p \in \Delta{\cal O}, \quad    \E_{o \sim p^*}[\ell(p^*, o, z)] \le \E_{o \sim p^*}[\ell(p, o, z)] + \varepsilon
     \end{align*}
\end{definition}

We now extend the definition of an omnipredictor. Let ${\cal C}$ be a collection of hypothesis from ${\cal X}$ to  $\Delta{\cal O}$, and ${\cal Z}$ be a collection of auxiliary information functions from ${\cal X}$ to ${\cal T}$.
\begin{definition}
    \label{def:omnipredictor-extended}
    Say that $\tilde{p} : {\cal X} \to \Delta {\cal O}$ is a $({\cal L}, {\cal C}, {\cal Z}, \varepsilon)$-omnipredictor if \[\E[\ell(\tilde{p}_i, o^*_i, z_i)] \le \E[\ell(c_i, o^*_i, z_i)] + \varepsilon\] for all $\ell \in {\cal L}$, $c \in {\cal C}$, and $z \in {\cal T}$.
\end{definition}
We show that our simple proof of omniprediction in the previous section can be easily adapted to accommodate the richer class of loss functions we consider. The main difference is that we now need the predictor $\tilde p$ to be multicalibrated w.r.t.\ ${\cal Z}$ and multiaccurate w.r.t.\ ${\cal C}||{\cal Z} = \{c||z : c \in {\cal C}, z\in {\cal Z}\}$.  
\begin{theorem}
    \label{thm:omniprediction-extended}
    Let ${\cal L}$ be a collection of $\varepsilon_3$-Bayes optimal loss functions. 
    If $\tilde{p}$ is $({\cal Z}, \varepsilon_1)$-multicalibrated and $({\cal C}||{\cal Z}, \varepsilon_2)$-multiaccurate, then $\tilde{p}$ is an $({\cal L}, {\cal C}, {\cal Z}, \varepsilon_1 + \varepsilon_2+\varepsilon_3)$-omnipredictor.
\end{theorem}

\begin{proof} First consider the case that $\varepsilon_1 = \varepsilon_2 = \varepsilon_3 = 0$. We have
    \begin{align*}
        \E[\ell(\tilde{p}_i, o^*_i, z_i)] 
        &= \E[\ell(\tilde{p}_i, \tilde{o}_i, z_i)] &\text{by Definition \ref{def:multicalibration} and $\tilde p$ is multicalibrated w.r.t.\ ${\cal Z}$,} \\
        &\le \E[\ell(c_i, \tilde{o}_i, z_i)] &\text{by Definition \ref{def:Bayes-optimal-loss},} \\
        &= \E[\ell(c_i, o^*_i, z_i)] &\text{by Lemma \ref{thm:multiaccuracy-omniprediction} and $\tilde p$ is multiaccruate w.r.t.\ ${\cal C}||{\cal Z}$.}
    \end{align*}
    In the general case, standard properties of statistical distance ensure that the first and third equality of expectation terms  differ by at most $\varepsilon_1$ and $\varepsilon_2$ (using the fact that the range of each $\ell \in {\cal L}$ is bounded between $0$~and~$1$). The right hand side of the second inequality has an additional $\varepsilon_3$ term by definition of Bayes optimality of $\cal L$. Therefore, we obtain  \[\E[\ell(\tilde{p}_i, o^*_i, z_i)] \le \E[\ell(c_i, o^*_i, z_i)] + \varepsilon_1 + \varepsilon_2+\varepsilon_3,\] so $\tilde{p}$ is an $({\cal L}, {\cal C}, {\cal Z}, \varepsilon_1 + \varepsilon_2+\varepsilon_3)$-omnipredictor.
\end{proof}

\section{Leakage Simulation and Outcome Indistinguishability}
\label{sec:leakage-simulation-and-oi}

The interesting relationship between multiaccuracy and the \emph{leakage simulation lemma}, or \emph{simulating auxiliary inputs} problem, in cryptography on the one hand allows us to obtain the first lower bound on the complexity of multiaccurate predictors. On the other hand, it inspires us to ask whether the stronger notion of multicalibration yields stronger consequences. We show this is the case, deriving a multicalibration-based proof of a hardcore lemma for {\em real-valued} functions.
 
Originating in the field of \emph{leakage-resilient cryptography}~\cite{dziembowski2008leakage}, the problem of leakage simulation defined by \cite{jetchev2014how} is as follows. Given correlated random variables $(X, O)$ on a set ${\cal X} \times {\cal O}$ and a collection of \emph{distinguisher functions} ${\cal A} = \{ A: {\cal X} \times {\cal O} \to \{0, 1\}\}$, the objective is to construct a low-complexity (w.r.t.\ ${\cal A}$) \emph{simulator} $h : {\cal X} \to \Delta{\cal O}$ such that no function in ${\cal A}$ can distinguish a sample $(X, O)$ from the true joint distribution from a simulated sample $(X, \tilde O)$, where $X$ is sampled from the true marginal distribution over ${\cal X}$ and $\tilde O$ is sampled according to the simulated distribution $h(X)$.

Observe that the leakage simulation problem can also be viewed as an equivalent reformulation of the problem of constructing a predictor satisfying \emph{no-access outcome indistinguishability} proposed in \cite{dwork2021outcome}, which they showed is equivalent to {\em multiaccuracy}. In the most general form, outcome indistinguishability studies a family ${\cal A}$ of \emph{distinguishers}, which are functions $A : {\cal X} \times {\cal O} \times (\Delta {\cal O})^{\cal X} \to [0, 1]$ that take as input an individual, an outcome, and a predictor and attempts to distinguish genuine outcomes $o^*_i$ from modeled outcomes $\tilde{o}_i$. The definition of outcome indistinguishability requires that the distinguishing advantage shall be small. Formally, 

\begin{definition}
\label{def:oi}
A predictor $\tilde{p} : {\cal X} \to \Delta{\cal O}$ is \emph{$({\cal A}, \varepsilon)$-outcome-indistinguishable} if for all $A \in {\cal A}$, \[\Big|\E\left[A(i, \tilde{o}_i, \tilde{p})\right] - \E\left[A(i, o^*_i, \tilde{p})\right]\Big| \le \varepsilon.\]
\end{definition}

In the above definition, the distinguisher $A$ has white-box access to the predictor $\tilde p$, which gives the most information of $\tilde p$. By restricting the access of $A$ to the predictor $\tilde p$ in different manners, Dwork \emph{et al.} \cite{dwork2021outcome} obtained a hierarchy of definitions of outcome indistinguishability. Of particular interest to us will be their notions of \emph{no-access} outcome indistinguishability and \emph{sample-access} outcome indistinguishability, which are shown to be equivalent to multiaccuracy and multicalibration respectively. In no-access outcome indistinguishability, the distinguisher cannot access $\tilde{p}$ at all, that is, every distinguisher $A \in {\cal A}$ takes the form $A(j, o, \tilde{p}) = A'(j, o)$. In \emph{sample-access} outcome indistinguishability, the distinguisher is only provided the output of $\tilde{p}$ on the individual under consideration---equivalently, each distinguisher $A \in {\cal A}$ takes the form $A(j, o, \tilde{p}) = A'(j, o, \tilde{p}_j)$ for some function $A' : {\cal X} \times {\cal O} \times (\Delta {\cal O}) \to [0, 1]$.

We observe that no access outcome indistinguishability is equivalent to leakage simulation: $(i, o^*_i)$ is the analogue of $(X, O)$; the predictor $\tilde p$ is analogous to the simulator $h$ while $\tilde o_i$  is analogous to $\tilde O$ (sampled according to $h(X)$ and $\tilde p_i$ respectively). The goals are also identical: no distinguisher $A$ in the class consider can tell part $(i,o^*_i)$ from $(i, \tilde o_i)$, or $(X, O)$ from $(x, \tilde O)$. 
In addition, algorithms in~\cite{trevisan2009regularity,jetchev2014how} for achieving leakage simulation are very similar to the multiaccuracy algorithm in~\cite{hebertjohnson2018multicalibration}.

Leveraging this equivalence immediately yields the first lower bound for the complexity of no-access outcome indistinguishability predictors relative to ${\cal A}$.  This follows from the result  in~\cite{chen2018complexity} that, relative to $\cal A$, the  complexity of a simulator is at least $\Omega(\ell \varepsilon^{-2})$, namely, the simulator makes at last $\Omega(\ell \varepsilon^{-2})$ black-box calls to some distinguishers in $\cal A$. 
Note that the lower bound only holds for simulators that are restricted to black-box use of the distinguishers and satisfy a restriction that, when invoked on input $X$, they only make black-box calls to the distinguishers on the same input $X$. All the leakage simulation, multiaccuracy, and multicalibration algorithms in the literature satisfy this restriction.  
Therefore, we arrive at that predictors satisfying no-access outcome indistinguishability (under same constraints) also have relative complexity $\Omega(\ell \varepsilon^{-2})$ w.r.t.\ distinguishers in $\cal A$. The equivalence between no-access outcome indistinguishability and multiaccuracy further tells us that the same relative complexity lower bound holds for multiaccurate predictors w.r.t.\ $\cal C$ (the analogue of $\cal A$). Finally, since multicalibration is stronger, the same lower bound extends to multicalibrated predictors.

This forecloses the existence of predictors that are smaller than the distinguisher yet fools them all (subject to the above restriction).  In other words, there is no predictor analogous to a pseudo-random generator that fools all polynomial-time tests.

Beyond the lower bound, the connection between no-access outcome indistinguishability and leakage simulation inspired us in two more directions. First, the work of~\cite{chen2018complexity} presented a leakage simulation algorithm via no-regret learning. Inspired by their algorithm, we present in Section~\ref{sec:no-regret-and-complexity} a general algorithmic framework for achieving sample-access outcome indistinguishability, equivalently multi-calibration, also via no-regret learning. Our framework  unifies algorithms in prior works. Second, inspired by the connection between leakage simulation and the hardcore lemma for Boolean functions, we ask whether the stronger notion of multicalibration yields stronger consequences. Indeed, we present in Section~\ref{sec:hardcore} a multicalibration-based proof of hardcore lemma for real-valued functions.  

\section{Sample Complexity and No-Regret Learning}
\label{sec:no-regret-and-complexity}

Various notions of multi-group fairness, including multiaccuracy, multicalibration, strict multicalibration (Definitions~\ref{def:multiaccuracy}, \ref{def:multicalibration}, and \ref{def:strict-multicalibration}), and low-degree multicalibration~\cite{gopalan2022lowdegree} are implied by the notion of \emph{outcome indistinguishability} of \cite{dwork2021outcome} with respect to different classes of adversaries. Thus, to achieve these notions it suffices to design algorithms for \emph{outcome indistinguishability}. On this front, our contributions are twofold. First, in Section~\ref{sec:algorithms}, we present an algorithmic template that unifies prior algorithms for achieving outcome indistinguishability through the lens of no-regret learning (see Algorithms~\ref{alg:no-regret} and~\ref{alg:mirror-descent}). Second, we show in Section~\ref{sec:complexity} that an instantiation of our algorithmic template yields an improved upper bound on sample complexity for achieving multicalibration in the multiclass setting and for low-degree multicalibration.

\subsection{Outcome Indistinguishability via No-Regret Learning}
\label{sec:algorithms}

In this section, we present an algorithmic template that unifies two existing algorithms for achieving outcome indistinguishability (and hence multicalibration), through the lens of no-regret learning. In Section \ref{sec:complexity}, we show that an instantiation of the template yields algorithms with improved upper bounds on the sample complexity of multicalibration.

The two algorithms under consideration are similar in that both make iterative updates to an arbitrary initial predictor $\tilde{p}$. However, they differ in their implementations of the update rule. The first update rule selects $A \in {\cal A}$ that successfully distinguishes $\tilde{p}$ from $p^*$ with advantage $\varepsilon$, and makes an additive update to $\tilde{p}$ resembling projected gradient descent \cite{hebertjohnson2018multicalibration, dwork2021outcome}. The second update rule also selects a distinguisher, but instead updates $\tilde{p}$ in a multiplicative manner \cite{kim2019multiaccuracy}.

To establish the claimed connection, we will first show in Section \ref{sec:no-regret} that the described algorithmic template can be instantiated with any update rule based on an algorithm with a \emph{no-regret guarantee}. We will then discuss in Section \ref{sec:mirror-descent} how projected gradient descent and multiplicative weight updates can be viewed as instances of \emph{mirror descent}, an algorithm with exactly the required no-regret guarantee. One benefit of this unified presentation via no-regret learning is that prior works require separate analyses for the two algorithms, but we only need a single, very simple, proof, relying only on the no-regret guarantee. We will also examine the relative merits of using projected gradient descent versus multiplicative weight updates for this role. (In brief, multiplicative weight updates work better in the multiclass setting, but projected gradient descent is more robust to a poor initialization.)

\subsubsection{No-Regret Updates}
\label{sec:no-regret}

We first recall the general framework of no-regret online learning. Consider $T$ rounds of gameplay between two players, called the \emph{decision-maker} and the \emph{adversary}. In each round, the decision-maker chooses a \emph{mixed strategy} (i.e., a probability distribution) over a finite set ${\cal O}$ of available \emph{pure strategies}. In response, the adversary adaptively chooses a \emph{loss function}, which assigns an arbitrary numeric loss between $0$ and $1$ to each to each available pure strategy. At the end of the round, the decision-maker incurs a penalty equal to the expected loss of its chosen mixed strategy, while learning the entire description of the loss function (not just the penalty it incurred).

Intuitively, a decision-making algorithm satisfies a \emph{no-regret guarantee} if the overall expected loss of a decision-maker employing the algorithm is no worse than the penalty the decision-maker would have incurred by playing any particular pure strategy against the same sequence of loss functions.

We now state a formal definition of no-regret learning. For simplicity, we focus our attention on decision-making algorithms whose strategy in round $t + 1$ is completely determined by what happened in round $t$, but this assumption is easy to relax.

\begin{definition}
    A \emph{decision-making algorithm} is specified by a distribution $D^{(1)} \in \Delta {\cal O} $ and a function \({\sf update} : \Delta {\cal O} \times [0, 1]^{\cal O}\to \Delta {\cal O} \). We say it \emph{satisfies a $(T, \varepsilon)$-no-regret guarantee} if \[\frac{1}{T}\sum_{t=1}^T \E_{o \sim D^{(t)}}\left[L^{(t)}(o)\right] \le \frac{1}{T}\sum_{t=1}^T L^{(t)}(o^*) + \varepsilon\] for every pure strategy $o^* \in {\cal O}$, every sequence of loss functions $L^{(1)}, \ldots, L^{(T)} \in [0, 1]^{\cal O}$, and the sequence of distributions $D^{(t)} \in \Delta {\cal O}$ given by \(D^{(t + 1)} = {\sf update}(D^{(t)}, L^{(t)}).\)
\end{definition}

\begin{example} The \emph{projected gradient descent} and \emph{multiplicative weight update} rules are
\begin{align*}
    D^{(t+1)} &= {\rm proj}_{\Delta{\cal O}}\left(D^{(t)} - \eta L^{(t)}\right), \\
    D^{(t+1)}(o) &= \frac{D^{(t)}(o)e^{-\eta L^{(t)}(o)}}{\sum_{o' \in {\cal O}} D^{(t)}(o')e^{-\eta L^{(t)}(o')}},
\end{align*}
respectively, for a parameter $\eta > 0$ called the \emph{step size}.
\end{example}

We will state the standard no-regret guarantees for these update rules, along with appropriate initializations, in Section \ref{sec:mirror-descent}.

Algorithm \ref{alg:no-regret} shows how to achieve outcome indistinguishability via no-regret updates. Indeed, Algorithm \ref{alg:no-regret} can be viewed as running $|{\cal X}|$ instances of a no-regret algorithm in parallel. Each instance corresponds to one member $j$ of the population ${\cal X}$, and the distribution $\tilde{p}_j$ corresponds to a mixed strategy over $o \in {\cal O}$. The predicted probabilities $\tilde{p}_j(o)$ are refined over multiple rounds.

The most important question is how the loss function is chosen in each round. Toward the goal of outcome indistinguishability against ${\cal A}$, given the current predictor $\tilde{p}$, the algorithm finds a distinguisher $A \in {\cal A}$ that has relatively high distinguishing advantage with respect to $\tilde{p}$. Such an adversary naturally defines a loss for each individual-outcome pair as follows: \[L_j(o) = A(j, o, \tilde{p}).\] We emphasize that though the no-regret algorithm is run separately for each member of ${\cal X}$, the choice of the distinguisher (and hence the loss functions) depends on the \emph{entire} predictor $\tilde{p}$. Furthermore, it suffices to find a distinguisher with relatively high advantage as opposed to maximal advantage.

Finally, if ${\cal X}$ is very large, it may be infeasible to run a separate instance of the no-regret algorithm for each $j \in {\cal X}$. This is not a problem because the collection of values $\tilde{p}_j(o)$ is implicitly defined by the distinguishers $A$ found in each round, which in turn give an efficient representation of the predictor $\tilde{p}$ as a whole.

\begin{algorithm}
    \caption{Outcome Indistinguishability with No-Regret Updates}
    \label{alg:no-regret}
    \DontPrintSemicolon

    \SetKw{Return}{return}
    \SetKwProg{Def}{def}{:}{end}
    \SetKwFunction{Construct}{Construct-via-No-Regret}

    \Def{\Construct{$\tilde{p}, {\cal A}, \varepsilon, {\sf update}$}}{
        \For{$A \in {\cal A}$}{
            \tcp{check for $({\cal A}, \varepsilon)$-OI violation}
            \If{$\E\left[A(i, \tilde{o}_i, \tilde{p})\right] > \E\left[A(i, o^*_i, \tilde{p})\right] + \varepsilon$}{
                \For{$j \in {\cal X}$}{
                    \For{$o \in {\cal O}$}{
                        \tcp{define loss function}
                        $L_j(o) \gets A(j, o, \tilde{p})$ 
                    }
                    \tcp{no-regret update}
                    $\tilde{p}'_j \gets {\sf update}(\tilde{p}_j, L_j)$
                }
                \tcp{recurse}
                \Return \Construct{$\tilde{p}', {\cal A}, \varepsilon, {\sf update}$}
            }
        }
        \Return $\tilde{p}$
    }
\end{algorithm}

\begin{theorem}
\label{thm:oi-no-regret}
Suppose ${\cal A}$ is closed under negation, meaning that $A \in {\cal A}$ iff $1 - A \in {\cal A}$, and that the function ${\sf update}$ satisfies a $(T, \varepsilon)$-no-regret guarantee when initialized at $p \in \Delta{\cal O}$. Set $\tilde{p}_j = p$ for each $j \in {\cal X}$. Then \(\textsc{Construct-via-No-Regret}(\tilde{p}, {\cal A}, \varepsilon, {\sf update})\) returns an $({\cal A}, \varepsilon)$-outcome-indistinguishable predictor in under $T$ recursive calls.
\end{theorem}

\begin{proof}
Let $\tilde{p}^{(t)}$ denote the argument to the $t$\textsuperscript{th} recursive call (e.g., $\tilde{p}^{(1)} = \tilde{p}$), let $\tilde{o}^{(t)}_i \in {\cal O}$ be a random variable whose conditional distribution given $i$ is specified by $\tilde{p}_i$, and suppose toward a contradiction that $T$ or more recursive calls are made. The no-regret guarantee implies that \[\frac{1}{T}\sum_{t=1}^T\left(\E\left[A(i, \tilde{o}^{(t)}_i, \tilde{p}^{(t)})\right] - \E\left[A(i, o^*_i, \tilde{p}^{(t)})\right]\right) \le \varepsilon.\] By design, if the algorithm does not terminate in $T$ recursive calls, then each summand on the left is greater than $\varepsilon$, which leads to a contradiction. Thus, $\textsc{Construct-via-No-Regret}$ always returns some predictor in under $T$ recursive calls, which the stopping condition clearly ensures is $({\cal A}, \varepsilon)$-outcome-indistinguishable.
\end{proof}

\subsubsection{Mirror Descent Updates}
\label{sec:mirror-descent}

In this section, we explain how projected gradient descent and multiplicative weight updates fit into the framework of no-regret learning and compare their advantages when used as the update rule in Algorithm \ref{alg:no-regret}. These two implementations of Algorithm \ref{alg:no-regret} are typically analyzed by tracking a potential function measuring the ``divergence'' of a predictor $\tilde{p}$ from the ground truth $p^*$ as updates to $\tilde{p}$ are made. The mirror descent perspective that we adopt in this section will clarify which ``divergence'' functions can be used in such an argument to derive a no-regret guarantee, while giving convergence rates for the corresponding update rules.

\paragraph{Background}

To begin, we state without proof some basic properties about the two algorithms under consideration. For more detail, we refer the reader to texts on convex optimization \cite{bubeck2015convex, nemirovsky1983problem}.

\begin{itemize}
\item \textbf{Projected Gradient Descent} Let $K \subset \mathbb{R}^n$ be a convex and compact constraint set. Then for any initialization $x_1 \in K$ and sequence $y_1, y_2, \ldots \in \mathbb{R}^n$, the update rule \[x_{t+1} = {\rm proj}_K ( x_t - \eta y_t )\] with step size $\eta > 0$ satisfies the regret bound  \[\sum_{t=1}^T \langle y_t, x_t - x \rangle \le \frac{\|x - x_1\|_2^2}{2\eta} + \frac{\eta}{2}\sum_{t=1}^T \|y_t\|_2^2\] for all $T \in \mathbb{N}$ and $x \in K$. Here, $\|\cdot\|_p$ denotes the \emph{$p$-norm}.

\item \textbf{Multiplicative Weight Updates} Let $\Delta^{n-1} = \{w \in [0, 1]^n : \|w\|_1 = 1\}$. Then for any initialization $x_1 \in \Delta^{n-1} \cap \mathbb{R}_{>0}^n$ and sequence $y_1, y_2, \ldots \in \mathbb{R}^n$, the update rule \[x_{t+1, j} \propto x_{t, j} \exp(-\eta y_{t, j}), \qquad x_{t+1} \in \Delta^{n-1}\] with step size $\eta > 0$ satisfies the regret bound \[\sum_{t=1}^T \langle y_t, x_t - x \rangle \le \frac{D_{\rm KL}\left(x\|x_1\right)}{\eta} + \frac{\eta}{2}\sum_{t=1}^T\|y_t\|_\infty^2\] for all $T \in \mathbb{N}$ and $x \in \Delta^{n-1}$. Here, $D_{\rm KL}$ denotes the \emph{Kullback-Leibler divergence} and $\|\cdot\|_\infty$ denotes the \emph{maximum norm}.

\item \textbf{Mirror Descent} Let $\|\cdot\|$ and $\|\cdot \|_*$ be dual norms on a finite-dimensional vector space $V$ and its dual $V^*$, and let $D$ be the \emph{Bregman divergence associated with a mirror map that is $1$-strongly convex with respect to $\|\cdot\|$ on $\Omega \subseteq V$}. Let $K \subset \overline{\Omega}$ be a convex and compact constraint set. Then for any initialization $x_1 \in K \cap \Omega$ and sequence $y_1, y_2, \ldots \in V^*$, the update rule
\[x_{t+1} = \argmin_{x \in K \cap \Omega} \;\; \langle y_t, x \rangle + \frac{D(x, x_t)}{\eta}\] with step size $\eta > 0$ is well-defined and satisfies the regret bound \[\sum_{t=1}^T \langle y_t, x_t - x \rangle \le \frac{D(x, x_1)}{\eta} + \frac{\eta}{2}\sum_{t=1}^T \|y_t\|_*^2\] for all $T \in \mathbb{N}$ and $x \in K$. Here, $\langle y, x \rangle$ denotes the real number $y(x)$.

The general update rule may be interpreted as selecting $x_{t+1}$ that responds well to $y_t$ without moving too far away from $x_t$, as enforced by the penalty term $D(x_{t+1}, x_t)/\eta$.

\item \textbf{Relationships} One can show that projected gradient descent is a special case of mirror descent with $\|\cdot\| = \|\cdot\|_* = \|\cdot\|_2$ on $\mathbb{R}^n$ and $D(p, q) = \frac{1}{2}\|p - q\|_2^2$ on $\Omega = \mathbb{R}^n$. Similarly, one can show that the multiplicative weights algorithm is a special case of mirror descent with $\|\cdot\| = \|\cdot\|_1$ and $\|\cdot\|_* = \|\cdot\|_\infty$ on $\mathbb{R}^n$ and $K = \Delta^{n-1}$ and $D(p, q) = D_{\rm KL}(p\|q) = \sum_{j=1}^n p_j\log(p_j/q_j)$ on $\Omega = \mathbb{R}_{>0}^n$.

\end{itemize}

\paragraph{Application to OI}\label{sec:application-to-OI}

Using the above notation, let $K = \Delta{\cal O}$ and \(L = \max_{f : {\cal O} \to [0, 1]} \|f\|_*\). Algorithm \ref{alg:mirror-descent} specializes Algorithm \ref{alg:no-regret} to the case of mirror descent updates, which includes projected gradient descent and multiplicative weight updates. In the pseudocode for Algorithm \ref{alg:mirror-descent}, the expression $\E_{o \sim p}[A(j, o, \tilde{p})]$ should be read as $\sum_{o \in {\cal O}} p(o) A(j, o, \tilde{p})$.

\begin{algorithm}
    \caption{Outcome Indistinguishability with Mirror Descent Updates}
    \label{alg:mirror-descent}
    \DontPrintSemicolon

    \SetKw{Return}{return}
    \SetKwProg{Def}{def}{:}{end}
    \SetKwFunction{Construct}{Construct-via-Mirror-Descent}

    \Def{\Construct{$\tilde{p}, {\cal A}, \varepsilon, D, \eta$}}{
        \For{$A \in {\cal A}$}{
            \tcp{check for $({\cal A}, \varepsilon)$-OI violation}
            \If{$\E\left[A(i, \tilde{o}_i, \tilde{p})\right] > \E\left[A(i, o^*_i, \tilde{p})\right] + \varepsilon$}{
                \For{$j \in {\cal X}$}{
                    \tcp{mirror descent update}
                    $\tilde{p}'_j \gets \argmin_{p \in \Delta{\cal O}} \;\; \E_{o \sim p}\left[A(j, o, \tilde{p})\right] + D(p, \tilde{p}_j)/\eta$
                }
                \tcp{recurse}
                \Return \Construct{$\tilde{p}', {\cal A}, \varepsilon, D, \eta$}
            }
        }
        \Return $\tilde{p}$
    }
\end{algorithm}

\begin{theorem}
\label{thm:oi-mirror-descent}
$\textsc{Construct-via-Mirror-Descent}(\tilde{p}, {\cal A}, D, \varepsilon, \eta)$ with step size $\eta = \varepsilon/L^2$ returns an $({\cal A}, \varepsilon)$-outcome-indistinguishable predictor in under \[2\left(\frac{L}{\varepsilon}\right)^2\E[D(p^*_i, \tilde{p}_i)]\] recursive calls, where $p_i^*$ denotes the conditional distribution of $o^*_i$ given $i$.
\end{theorem}

\begin{proof}
Let $\tilde{p}^{(t)}$ denote the argument to the $t$\textsuperscript{th} recursive call (e.g., $\tilde{p}^{(1)} = \tilde{p}$), and let $\tilde{o}^{(t)}_i \in {\cal O}$ be a random variable whose conditional distribution given $i$ is specified by $\tilde{p}_i$. The mirror descent regret bound implies that \[\sum_{t=1}^T \left(\E\left[A(i, \tilde{o}^{(t)}_i; \tilde{p}^{(t)})\right] - \E\left[A(i, o^*_i; \tilde{p}^{(t)})\right] \right) \le \frac{\E[D(p_i^*, \tilde{p}_i)]}{\eta} + \frac{\eta T L^2}{2}.\] By design, each summand on the left is greater than $\varepsilon$, which leads to a contradiction if the step size is $\eta = \varepsilon/L^2$ and there are $T \ge 2(L/\varepsilon)^2\E[D(p^*_i, \tilde{p}_i)]$ calls. Thus, $\textsc{Construct-via-Mirror-Descent}$ always returns some predictor, which the stopping condition clearly ensures is $({\cal A}, \varepsilon)$-outcome-indistinguishable.
\end{proof}

\begin{example}
\label{example:additive-updates}
Using projected gradient descent yields $L = \sqrt{|{\cal O}|}$ and a bound of \[\frac{|{\cal O}|\E\left[\|p_i^* - \tilde{p}_i\|_2^2\right]}{\varepsilon^2} \le \frac{2|{\cal O}|}{\varepsilon^2}\] on the number of recursive calls. In the binary case, i.e., ${\cal O} = \{0, 1\}$, the update rule satisfies \[\tilde{p}'_j(1) \leftarrow {\rm proj}_{[0, 1]} \left( \tilde{p}_j(1) - \frac{\eta}{2}(A(j, 1, \tilde{p}) - A(j, 0, \tilde{p}))\right),\] which agrees with the original outcome indistinguishability algorithm \cite{dwork2021outcome}.
\end{example}

\begin{example}
\label{example:multiplicative-updates}
Using multiplicative weight updates yields $L = 1$ and a bound of $2\varepsilon^{-2}\E[D_{\rm KL}(p_i^* \| \tilde{p}_i)]$ on the number of recursive calls. If we initialize $\tilde{p}_j$ to the uniform distribution on ${\cal O}$ for all $j \in {\cal X}$, then this bound reduces to $2\varepsilon^{-2}\log(|{\cal O}|)$, which has a better dependence on $|{\cal O}|$ than the bound for projected gradient descent but allows less flexibility in the initialization of $\tilde{p}$.
\end{example}

The original analyses of the generic multicalibration and outcome indistinguishability algorithms are not phrased in terms of no-regret bounds like our proofs of Theorems \ref{thm:oi-no-regret} and \ref{thm:oi-mirror-descent}. Instead, they track changes to a \emph{potential function} as $\tilde{p}$ is iteratively updated. In our proof of Theorem \ref{thm:oi-mirror-descent}, this potential function corresponds exactly to the quantity $\E[D(p^*_i, \tilde{p}_i)]$.

\subsection{Weak Agnostic Learning}
\label{sec:weak-agnostic-learning}

To prepare for the discussion of our new complexity upper bounds, we now present a variant of Algorithm \ref{alg:no-regret} that abstracts the process of finding a distinguisher $A \in {\cal A}$ that distinguishes real from modeled outcomes with advantage $\varepsilon$. The abstraction we consider is based on that of a \emph{weak agnostic learner}, which appeared in the original paper on multi-group fairness \cite{hebertjohnson2018multicalibration}.

\begin{definition}
    \label{def:weak-agnostic-learner}
    Let ${\sf WAL}_{{\cal A}}$ be an algorithm that takes as input a parameter $\varepsilon > 0$ and a predictor $\tilde{p}$ and outputs either a distinguisher $A \in {\cal A}$ or the symbol $\bot$. We assume that ${\sf WAL}_{{\cal A}}$ has the ability to draw data samples that are i.i.d. copies of $(i, o^*_i)$. We say that ${\sf WAL}_{{\cal A}}$ is a \emph{weak agnostic learner} with failure probability $\beta$ if the following two conditions hold:
    \begin{itemize}
        \item If there exists $A \in {\cal A}$ such that $\Delta_A(\tilde{p}) := \E[A(i, \tilde{o}_i, \tilde{p})] - \E[A(i, o^*_i, \tilde{p})] > \varepsilon$, then ${\sf WAL}_{{\cal A}}(\varepsilon, \tilde{p})$ outputs $A' \in {\cal A}$ such that $\Delta_{A'}(\tilde{p}) > \varepsilon/2$ with probability at least $1 - \beta$.
        \item If every $A \in {\cal A}$ satisfies $\Delta_A(\tilde{p}) \le \varepsilon$, then ${\sf WAL}_{{\cal A}}(\varepsilon, \tilde{p})$ outputs either $A' \in {\cal A}$ such that $\Delta_{A'}(\tilde{p}) > \varepsilon/2$ or $\bot$ with probability at least $1 - \beta$.
    \end{itemize}
\end{definition}

A simple application of Hoeffding's inequality and a union bound yields the following lemma.

\begin{lemma}
    \label{thm:samples-per-iteration}
    For every family ${\cal A}$ of distinguishers, there exists a weak agnostic learning algorithm ${\sf WAL}_{{\cal A}}$ with failure probability $\beta$ that draws at most \(O\left(\log(|{\cal A}|/\beta)/\varepsilon^2\right)\) on input $\varepsilon$ and $\tilde{p}$.
\end{lemma}

Algorithm \ref{alg:oi-via-weak-agnostic-learning} shows how to utilize $\mathsf{WAL}_{{\cal A}}$ as a subroutine for achieving outcome indistinguishability.

\begin{algorithm}[tb]
    \caption{Outcome Indistinguishability via Weak Agnostic Learning}
    \label{alg:oi-via-weak-agnostic-learning}
    \DontPrintSemicolon
    \SetKw{Return}{return}
    \SetKwProg{Def}{def}{:}{end}
    \SetKwFunction{Construct}{Construct-via-WAL}
    \vspace{0.2cm}

    \Def{\Construct{$\varepsilon, \tilde{p}, \mathsf{WAL}_{{\cal A}}, {\sf update}$}}{
        $A \gets \mathsf{WAL}_{{\cal A}}(\varepsilon, \tilde{p})$ \tcp*{check for $({\cal A}, \varepsilon)$-OI violation}
        \If{$A = \bot$}{
            \Return $\tilde{p}$\;
        }
        \For{$j \in {\cal X}$}{
            \For{$o \in {\cal O}$}{
                $L_j(o) \gets A(j, o, \tilde{p})$ \tcp*{define loss function}
            }
            $\tilde{p}'_j \gets {\sf update}(\tilde{p}_j, L_j)$ \tcp*{no-regret update}
        }
        \Return \Construct{$\varepsilon, \tilde{p}', \mathsf{WAL}_{{\cal A}}, {\sf update}$} \tcp*{recurse}
    }
    
\end{algorithm}

\begin{theorem}
    \label{thm:generic-sample-complexity}
    If ${\cal A}$ is closed under complement, then there exists an algorithm ${\sf WAL}_{{\cal A}}$, an algorithm ${\sf update}$, and an initialization of $\tilde{p}$ such that with probability at least $1-\beta$, the algorithm  \(\textsc{Construct}(\varepsilon, \tilde{p}, {\sf WAL}_{{\cal A}}, {\sf update})\) returns an $({\cal A}, \varepsilon)$-outcome-indistinguishable predictor in at most $\log(\ell)/\varepsilon^2$ recursive calls and using at most $\log(\ell)\log(|{\cal A}|/\beta)/\varepsilon^4$ samples, where $\ell = |{\cal O}|$.
\end{theorem}

\begin{proof}
    The proof of Theorem \ref{thm:oi-mirror-descent} and Example \ref{example:multiplicative-updates}, along with the two properties in Definition \ref{def:weak-agnostic-learner}, gives a $O(\log(\ell)/\varepsilon^2)$ upper bound on the number of updates. Lemma \ref{thm:samples-per-iteration} upper bounds the number of samples needed per update. Choosing $\beta$ appropriately and applying a union bound over the entire sequence of updates yields the claimed sample complexity upper bound.
\end{proof}

\subsection{Sample Complexity}
\label{sec:complexity}

In this section, we consider a particular instantiation (Algorithm~\ref{alg:outcome-indistinguishability}) of our algorithmic template from Section~\ref{sec:algorithms} to derive improved upper bounds on sample complexity for achieving multicalibration in the multiclass setting and for low-degree multicalibration.

\begin{algorithm}[tb]
    \caption{Multiclass Outcome Indistinguishability via Multiplicative Weight Updates}
    \label{alg:outcome-indistinguishability}
    \DontPrintSemicolon
    \SetKw{Break}{break}
    \SetKw{Not}{not}
    \SetKw{And}{and}
    \SetKw{True}{true}
    \SetKw{False}{false}
    
    \vspace{0.2cm}
    \KwData{distinguisher family ${\cal A}$, parameter $\varepsilon$, iteration count $t$, samples per iteration $n$}
    \KwResult{predictor $\tilde{p} : {\cal X} \to \Delta{\cal O}$}
    \vspace{0.1cm}
    $\tilde{p}^{(0)}_j(o) \gets 1/|{\cal O}|$ for all $j \in {\cal X}$ and $o \in {\cal O}$ \tcp*{initialize $\tilde{p}^{(0)}$ to constant}
    \For{$s = 0, 1, \ldots, t - 1$}{
        $(i_{s1}, o^*_{i_{s1}}), \ldots, (i_{sn}, o^*_{i_{sn}}) \sim (i, o^*_i)$ \tcp*{draw $n$ fresh iid samples}
        ${\rm changed} \gets \False$\;
        \For{$A \in {\cal A}$}{
            $\alpha \gets \sum_{m=1}^n A(i_{sm}, o^*_{i_{sm}}, \tilde{p}^{(s)})$ \tcp*{genuine outcomes}
            $\beta \gets \sum_{m=1}^n\sum_{o \in {\cal O}}\tilde{p}_{i_{sm}}(o)A(i_{sm}, o, \tilde{p}^{(s)})$ \tcp*{modeled outcomes}
            \If(\tcp*[f]{is $A$'s advantage large?}){\Not {\rm changed} \And $|\beta - \alpha| > \varepsilon n / 2$}{   
                $\gamma \gets \mathrm{sign}(\beta - \alpha) \cdot \varepsilon/3$ \tcp*{step size}
                \For{$j \in {\cal X}$}{
                    \For{$o \in {\cal O}$}{
                        $f_j(o) \gets \tilde{p}^{(s)}_j(o)\exp(-\gamma A(j, o, \tilde{p}^{(s)}))$\tcp*{update weights}
                    }
                    $\tilde{p}^{(s+1)}_j(o) \gets f_j(o) / \sum_{o' \in {\cal O}} f_j(o')$\;
                    ${\rm changed} \gets \True$\;
                }
            }
        }
        \lIf{\Not {\rm changed}}{\Return $\tilde{p}^{(s)}$}
    }
    \Return $\bot$\;
\end{algorithm}

\begin{lemma}
    \label{thm:oi-iterations}
    Running Algorithm~\ref{alg:outcome-indistinguishability} with appropriately chosen $t \lesssim \log(|{\cal O}|)/\varepsilon^2$ and $n \lesssim \log(|{\cal A}|t)/\varepsilon^2$ yields an $({\cal A}, \varepsilon)$-outcome-indistinguishable predictor with probability at least $99\%$.\footnote{By $f(x) \lesssim g(x)$, we mean that $f(x) = O(g(x))$.}
\end{lemma}

\begin{proof}
Throughout this proof, and for the remainder of this section, let \[\Delta_A(\tilde{p}) = \Big|\E\left[A(i, \tilde{o}_i, \tilde{p})\right] - \E\left[A(i, o^*_i, \tilde{p})\right]\Big|\] for a distinguisher $A \in {\cal A}$ and predictor $\tilde{p} : {\cal X} \to \Delta{\cal O}$.We say that Algorithm~\ref{alg:outcome-indistinguishability} \emph{performs an update} in iteration $s$ if it reaches the line ``${\rm changed} \gets {\bf true}$'' during iteration $s$ of the outermost for-loop. We claim that the ``exhaustive search'' over $A \in {\cal A}$ in the algorithm correctly implements the weak agnostic learning step described in Section~\ref{sec:weak-agnostic-learning}. Indeed, a standard application of a Chernoff bound and union bound shows that for an appropriately chosen number $n \lesssim \log(|{\cal A}|t)/\varepsilon^2$ of samples per iteration, the following two properties hold with probability at least $99\%$ across all iterations $s \in \{0, 1, \ldots, t-1\}$:
\begin{itemize}
    \item[(a)] If some $A \in {\cal A}$ satisfies $\Delta_{A}(\tilde{p}^{(s)}) > \varepsilon$, then the algorithm performs an update in iteration $s$.
    \item[(b)] If the algorithm performs an update in iteration $s$ using $A^{(s)} \in {\cal A}$, then $\Delta_{A^{(s)}}(\tilde{p}^{(s)}) > \varepsilon/3$.
\end{itemize}
Since the algorithm only outputs $\tilde{p} \neq \bot$ after an iteration when no update was performed, property (a) immediately implies that such a predictor $\tilde{p}$ must be $({\cal A}, \varepsilon)$-outcome-indistinguishable. It remains to show that Algorithm~\ref{alg:outcome-indistinguishability} will never output $\bot$ for an appropriate value $t \lesssim \log(|{\cal O}|)/\varepsilon^2$. However, this follows immediately from property (b) and Example~\ref{example:multiplicative-updates}.
\end{proof}

By choosing the distinguisher family ${\cal A}$ judiciously, we can achieve multicalibration in a more sample-efficient manner than existing algorithms. In fact, the construction of the family ${\cal A}$ follows naturally from our statistical distance-based definition of multicalibration:

\begin{definition}
\label{def:oi-family-for-multicalibration}
Let \({\cal A}_{{\cal C}, {\cal G}}^{\mathrm{MC}} = \left\{ A_{c,E} \mid c \in {\cal C}, \; E \subseteq {\cal Y} \times {\cal O} \times {\cal G}\right\},\) where \[A_{c,E}(j, o, \tilde{p}) = {\bf 1}[(c_j, o, \hat{p}_j) \in E]\] for each member $j \in {\cal X}$, possible outcome $o \in {\cal O}$, and predictor $\tilde{p} : {\cal X} \to \Delta{\cal O}$.
\end{definition}

\begin{theorem}
\label{thm:multicalibration-samples}
Running Algorithm~\ref{alg:outcome-indistinguishability} on ${\cal A} = {\cal A}_{{\cal C}, {\cal G}}^{\mathrm{MC}}$ and appropriately chosen ${\cal G}, t, n$ yields a predictor $\tilde{p}$ such that $\hat{p}$ is $({\cal C}, \varepsilon)$-multicalibrated with probability at least $99\%$. The algorithm samples at most \[tn \lesssim \left(\log(|{\cal C}|) + |{\cal Y}||{\cal O}|\left(\frac{4}{\varepsilon}\right)^{|{\cal O}|-1}\right) \cdot \frac{\log(|{\cal O}|)}{\varepsilon^4}\] i.i.d. individual-outcome pairs.
\end{theorem}

Letting ${\cal Y} = \{0, 1\}$ in Theorem~\ref{thm:multicalibration-samples}, so that ${\cal C} \subseteq \{0, 1\}^{\cal X}$, we recover the sample complexity upper bound that we initially stated in Section~\ref{sec:introduction}. Intuitively, our savings compared to prior works comes from the fact that our algorithm \emph{directly} targets the milder requirements of ordinary multicalibration, while prior algorithms typically ``go through'' strict multicalibration by aiming for stringent per-level-set guarantees.

\begin{proof}[Proof of Theorem~\ref{thm:multicalibration-samples}]
Observe that the family ${\cal A}^{\rm MC}_{{\cal C}, {\cal G}}$ has size $2^{|{\cal Y}||{\cal O}||{\cal G}|}|{\cal C}|$. By Lemma~\ref{thm:grid-size}, we can choose an $\eta$-covering ${\cal G}$ of $\Delta{\cal O}$ in such a way that $|{\cal G}| < (3/\eta)^{\ell-1}$. Thus, by Lemma~\ref{thm:oi-iterations}, Algorithm~\ref{alg:outcome-indistinguishability} with input ${\cal A}^{\rm MC}_{{\cal C}, {\cal G}}$ and parameters $\varepsilon/4$ and $\eta = 3\varepsilon/4$ outputs a $({\cal A}^{\rm MC}_{{\cal C}, {\cal G}}, \varepsilon/4)$-outcome-indistinguishable predictor $\tilde{p}$ with probability at least $99\%$ using at most \[tn \lesssim \left(\log(|{\cal C}|) + |{\cal Y}||{\cal O}|\left(\frac{4}{\varepsilon}\right)^{|{\cal O}|-1}\right) \cdot \frac{\log(|{\cal O}|)}{\varepsilon^4}\] samples. Since ${\cal G}$ is an $\eta$-covering of $\Delta{\cal O}$, the inequality $\delta(\hat{o}_i, \tilde{o}_i \mid i) \le \eta$ holds almost surely, so \[\delta\big((c_i, \hat{o}_i, \hat{p}_i), (c_i, o^*_i, \hat{p}_i)\big) \le \delta\big((c_i, \tilde{o}_i, \hat{p}_i), (c_i, o^*_i, \hat{p}_i)\big) + \eta.\] The first term is at most $\varepsilon/4$ by the definition of statistical distance and $({\cal A}_{\cal C, G}^{\mathrm{MC}}, \varepsilon/4)$-outcome-indistinguishability of $\tilde{p}$, so the right hand side is at most $\varepsilon$. We conclude that the discretized predictor $\hat{p}$ is $({\cal C}, \varepsilon)$-multicalibrated.
\end{proof}

To justify our upper bound for degree-$k$ multicalibration in Section~\ref{sec:introduction}, we now turn our attention to the notion of \emph{low-degree} multicalibration from \cite{gopalan2022lowdegree}. A rephrased statement of the definition is as follows:

\begin{definition}
    A function $f : [0, 1]^\ell \to [0, 1]$ is a \emph{monomial of degree less than $k$} if it takes the form $f(v) = v_{t_1} \cdots v_{t_j}$ for some $j < k$ indices $t_1, \ldots, t_j \in [\ell]$. For ${\cal O} = \{0, 1\}$ and ${\cal C} \subseteq [0, 1]^{\cal X}$ and $k \in \mathbb{N}$, let ${\cal A}_{{\cal C}, k}^{\mathrm{MC}}$ to be the family of all distinguishers of the form \[A_{c,f,o'}(j, o, \tilde{p}) = c(j)f(\tilde{p}_j){\bf 1}[o = o'],\] where $c \in {\cal C}$ and $o' \in {\cal O}$ and $f$ is a monomial of degree less than $k$. We say that a predictor $\tilde{p}$ is $({\cal C}, \varepsilon)$-degree-$k$ multicalibrated if $\tilde{p}$ is $({\cal A}_{{\cal C}, k}^{\mathrm{MC}}, \varepsilon)$-outcome indistinguishable.
\end{definition}

Using the fact that the family ${\cal A}_{{\cal C}, k}^{\mathrm{MC}}$ is a subset of ${\cal A}_{{\cal C}, {\cal G}}^{\mathrm{MC}}$, one can show that degree-$k$ multicalibration is \emph{weaker} than the notion of multicalibration considered so far. With this in mind, it should not be surprising that running Algorithm~\ref{alg:outcome-indistinguishability} on ${\cal A}_{{\cal C}, k}^{\mathrm{MC}}$ instead of ${\cal A}_{{\cal C}, {\cal G}}^{\mathrm{MC}}$ immediately gives us the following tighter upper bound on the samples needed for degree-$k$ multicalibration:

\begin{theorem}
    \label{thm:low-degree-samples}
    Running Algorithm~\ref{alg:outcome-indistinguishability} on ${\cal A} = {\cal A}_{{\cal C}, k}^{\mathrm{MC}}$ and appropriately chosen $t, n$ yields a $({\cal C}, \varepsilon)$-degree-$k$ multicalibrated predictor with probability at least $99\%$ and samples at most \[tn \lesssim \left(\log(|{\cal C}|) + k\log\left(\frac{\ell}{k}\right) + \log\left(\frac{\ell}{\varepsilon}\right)\right) \cdot \frac{\log(\ell)}{\varepsilon^4}\] i.i.d. individual-outcome pairs, where $\ell = |{\cal O}|$ and $\ell \ge k$.
\end{theorem}

The proof of Theorem~\ref{thm:low-degree-samples} will show, in particular, that the improvement in our Theorem~\ref{thm:low-degree-samples} compared to Theorem 35 of \cite{gopalan2022lowdegree} comes primarily from our deliberate use of \emph{multiplicative} updates to $\tilde{p}$ in Algorithm~\ref{alg:outcome-indistinguishability}, as opposed to additive updates.

\begin{proof}[Proof of Theorem~\ref{thm:low-degree-samples}]
A standard counting argument shows that the family ${\cal A}^{\rm MC}_{{\cal C}, k}$ has size $\binom{k+\ell-1}{\ell-1}\ell|{\cal C}|$. Thus, by Lemma~\ref{thm:oi-iterations}, Algorithm~\ref{alg:outcome-indistinguishability} with input ${\cal A}^{\rm MC}_{{\cal C}, {\cal G}}$ and parameters $\varepsilon/4$ and $\eta = 3\varepsilon/4$ outputs a $({\cal A}^{\rm MC}_{{\cal C}, k}, \varepsilon)$-outcome-indistinguishable predictor $\tilde{p}$ with probability at least $99\%$ using at most \[tn \lesssim \left(\log(|{\cal C}|) + k\log\left(\frac{\ell}{k}\right) + \log\left(\frac{\ell}{\varepsilon}\right)\right) \cdot \frac{\log(\ell)}{\varepsilon^4}\] samples. By definition, such a predictor $\tilde{p}$ is $({\cal C}, \varepsilon)$-degree-$k$ multicalibrated.
\end{proof}

To conclude this section, we also give an upper bound on the sample complexity of \emph{strict} multicalibration.

\begin{definition}
Let \({\cal A}_{{\cal C}, {\cal G}}^{\mathrm{SMC}} = \left\{ A_{\vec{c},E} \mid \vec{c} \in {\cal C}^{\cal G}, \; E \subseteq {\cal Y} \times {\cal O} \times {\cal G}\right\},\) where \[A_{\vec{c},E}(j, o, \tilde{p}) = {\bf 1}[(\vec{c}(\hat{p}_j)_j, o, \hat{p}_j) \in E]\] for each member $j \in {\cal X}$, possible outcome $o \in {\cal O}$, and predictor $\tilde{p} : {\cal X} \to \Delta{\cal O}$.
\end{definition}

\begin{theorem}
\label{thm:strict-multicalibration-samples}
Running Algorithm~\ref{alg:outcome-indistinguishability} on ${\cal A} = {\cal A}_{{\cal C}, {\cal G}}^{\mathrm{SMC}}$ and appropriately chosen ${\cal G}, t, n$ yields a predictor $\tilde{p}$ such that $\hat{p}$ is strictly $({\cal C}, \varepsilon)$-multicalibrated with probability at least $99\%$. The algorithm samples at most \[tn \lesssim \left(\log(|{\cal C}|) + |{\cal Y}||{\cal O}|\right) \cdot \left(\frac{4}{\varepsilon}\right)^{|{\cal O}|-1} \cdot \frac{\log(|{\cal O}|)}{\varepsilon^4}\] i.i.d. individual-outcome pairs.
\end{theorem}

\begin{proof}
Observe that the family ${\cal A}^{\rm SMC}_{{\cal C}, {\cal G}}$ has size $2^{|{\cal Y}||{\cal O}||{\cal G}|}|{\cal C}|^{|{\cal G}|}$. By Lemma~\ref{thm:grid-size}, we can choose an $\eta$-covering ${\cal G}$ of $\Delta{\cal O}$ in such a way that $|{\cal G}| < (3/\eta)^{\ell-1}$. Thus, by Lemma~\ref{thm:oi-iterations}, Algorithm~\ref{alg:outcome-indistinguishability} with input ${\cal A}^{\rm SMC}_{{\cal C}, {\cal G}}$ and parameters $\varepsilon/4$ and $\eta = 3\varepsilon/4$ outputs a $({\cal A}^{\rm SMC}_{{\cal C}, {\cal G}}, \varepsilon/4)$-outcome-indistinguishable predictor $\tilde{p}$ with probability at least $99\%$ using at most \[tn \lesssim \left(\log(|{\cal C}|) + |{\cal Y}||{\cal O}|\right) \cdot \left(\frac{4}{\varepsilon}\right)^{|{\cal O}|-1} \cdot \frac{\log(|{\cal O}|)}{\varepsilon^4}\] samples. Since ${\cal G}$ is an $\eta$-covering of $\Delta{\cal O}$, the inequality $\delta(\hat{o}_i, \tilde{o}_i \mid i) \le \eta$ holds almost surely, so \[\E\left[\max_{c \in {\cal C}} \, \delta \big( (c_i, \hat{o}_i), (c_i, o^*_i) \mid \hat{p}_i \big)\right] \le \E\left[\max_{c \in {\cal C}} \, \delta \big( (c_i, \tilde{o}_i), (c_i, o^*_i) \mid \hat{p}_i \big)\right] + \eta.\] The expectation on the right hand side is at most $\varepsilon/4$ by the definition of statistical distance and $({\cal A}_{\cal C, G}^{\mathrm{SMC}}, \varepsilon/4)$-outcome-indistinguishability of $\tilde{p}$, so the right hand side is at most $\varepsilon$. We conclude that the discretized predictor $\hat{p}$ is strictly $({\cal C}, \varepsilon)$-multicalibrated.
\end{proof}

\subsection{Improvements in Special Cases}

In the case that ${\cal C} \subseteq \{0, 1\}^{\cal X}$, we can refine Theorem~\ref{thm:multicalibration-samples} so that its bound depends on the \emph{Vapnik–Chervonenkis dimension} ${\sf VC}({\cal C})$ of ${\cal C}$ instead of the logarithm of the cardinality of ${\cal C}$. To do so, we first state some definitions and lemmas that follow directly from basic properties of VC dimension that can be found in standard texts \cite{shalevshwartz2014understanding}.

\begin{definition}
    For finite sets ${\cal S}$ and ${\cal H} \subseteq \{0, 1\}^{\cal S}$, the \emph{VC dimension} ${\sf VC}({\cal H})$ of ${\cal H}$ is the size of the largest subset ${\cal T} \subseteq {\cal S}$ such that every possible function ${\cal T} \to \{0, 1\}$ is the restriction of some function in ${\cal H}$. Such a set ${\cal T}$ is said to be \emph{shattered} by ${\cal H}$. Given a family ${\cal A}$ of distinguishers and a predictor $\tilde{p}$, we write ${\sf VC}({\cal A}_{\tilde{p}})$ to denote the VC dimension of the collection ${\cal A}_{\tilde{p}} \subseteq \{0, 1\}^{{\cal X} \times {\cal O}}$ of functions $A_{\tilde{p}} : {\cal X} \times {\cal O} \to \{0, 1\}$ given by \(A_{\tilde{p}}(j, o) = A(j, o, \tilde{p})\) for $A \in {\cal A}$.
\end{definition}

\begin{lemma}
    \label{thm:vc-samples-per-iteration}
    For every family ${\cal A}$ of distinguishers and every $\varepsilon > 0$, there exists a weak agnostic learning algorithm ${\sf WAL}_{{\cal A}, \varepsilon}$ with failure probability $\beta$ that draws at most \(O\left(({\sf VC}({\cal A}_{\tilde{p}}) + \log(1/\beta))/\varepsilon^2\right)\) samples.
\end{lemma}

With these tools in hand, we now state and prove the main result of this section.

\begin{theorem}
    \label{thm:multicalibration-vc-samples}
    Fix ${\cal C} \subseteq \{0, 1\}^{\cal X}$ and let $\ell = |{\cal O}|$. There is an algorithm that takes as input \[O\left({\sf VC}({\cal C}) + \ell\left(\frac{4}{\varepsilon}\right)^{\ell-1}\right) \cdot \frac{\log(\ell)}{\varepsilon^4}\] i.i.d. individual-outcome pairs and outputs a strictly $({\cal C}, \varepsilon)$-multicalibrated predictor w.p. $99\%$.
\end{theorem}

\begin{proof}
    Given ${\cal C}$, ${\cal G}$, and $E \subseteq \{0, 1\} \times {\cal O} \times {\cal G}$, let \({\cal A}_{{\cal C}, {\cal G}, E}^{\mathrm{MC}} = \left\{ A_{c,E} \mid c \in {\cal C}\right\}\) where \[A_{c,E}(j, o, \tilde{p}) = {\bf 1}[(c_j, o, \hat{p}_j) \in E]\] for each member $j \in {\cal X}$, possible outcome $o \in {\cal O}$, and predictor $\tilde{p} : {\cal X} \to \Delta{\cal O}$. Note that \[{\cal A}_{{\cal C}, {\cal G}, \tilde{p}}^{\rm MC} = \bigcup_E {\cal A}_{{\cal C}, {\cal G}, E, \tilde{p}}^{\mathrm{MC}},\] where the union is taken over all $2^{2|{\cal O}||{\cal G}|}$ possible choices of $E$. We claim that \[{\sf VC}\left({\cal A}_{{\cal C}, {\cal G}, E, \tilde{p}}^{\mathrm{MC}}\right) \le {\sf VC}({\cal C)}\] for any $E$ and any predictor $\tilde{p}$. Once this is shown, it will follow from Lemma \ref{thm:vc-samples-per-iteration} with failure probability $\beta/2^{2|{\cal O}||{\cal G}|}$ and a union bound over the choice of $E$ that there exists a weak agnostic learning algorithm ${\sf WAL}_{{\cal A}_{{\cal C}, {\cal G}}^{\mathrm{MC}}, \varepsilon}$ with failure probability $\beta$ that uses only $O\big(({\sf VC}({\cal C}) + \log(2^{2|{\cal O}||{\cal G}|}/\beta))/\varepsilon^2\big)$ samples. Proceeding as in the proof of Lemma~\ref{thm:generic-sample-complexity} and Theorem~\ref{thm:multicalibration-samples} yields the desired sample complexity bound.
    
    It remains to prove the claim. To this end, fix $E$ and $\tilde{p}$ and let ${\cal T} = \{(j_1, o_1), \ldots, (j_d, o_d)\} \subseteq {\cal X} \times {\cal O}$ be a maximum size set shattered by ${\cal A}_{{\cal C}, {\cal G}, E, \tilde{p}}^{\rm MC}$. Consider the element $(j_d, o_d) \in {\cal T}$. In order for there to exist two distinguishers $A_{c, E}$ and $A_{c', E}$ such that $A_{c,E}(j_d, o_d, \tilde{p}) = 0$ and $A_{c',E}(j_d, o_d, \tilde{p}) = 1$, it must be the case that $c(j_d) \neq c'(j_d)$, or else we would have $A_{c,E}(j_d, o_d, \tilde{p}) = A_{c', E}(j_d, o_d, \tilde{p})$. More generally, if the restrictions of $A_{c,E}$ and $A_{c',E}$ to ${\cal T}$ are distinct, then the restrictions of $c$ and $c'$ to $\{j_1, \ldots, j_d\}$ must also be distinct. Since the distinguishers shatter ${\cal T}$, it follows that $\{j_1, \ldots, j_d\} \subseteq {\cal X}$ is shattered by ${\cal C}$ and hence that $d \le {\sf VC}({\cal C})$. This completes the proof of the claim.
\end{proof}

\subsection{A Randomized Approach}

In this section, we give a variant of Theorem~\ref{thm:multicalibration-samples} using an alternate algorithm for achieving multicalibration in the case that ${\cal C} \subseteq \{0, 1\}^{\cal X}$ that we believe may be of interest. We will present the algorithm in this section with a slightly simpler notion of weak agnostic learning than in the preceding sections. We caution the reader that the sample complexity upper bound we derive in Theorem~\ref{thm:multicalibration-samples-random-signs} will not be as tight as that of Theorem~\ref{thm:multicalibration-samples}.

\begin{definition}
    \label{def:weak-agnostic-learner-standard}
    Suppose ${\cal C} \subseteq \{0, 1\}^{\cal X}$. Let ${\sf WAL}_{{\cal C}}$ be an algorithm that takes as input a parameter $\varepsilon > 0$ and a sequence of labeled data $(x_1, y_1), (x_2, y_2), \ldots \in {\cal X} \times [-1, 1]$, and outputs either a function $c \in {\cal C}$ or the symbol $\bot$. Consider a fixed joint distribution over ${\cal X} \times [-1, 1]$, and a pair $(x, y)$ drawn from this distribution. We say that ${\sf WAL}_{{\cal C}}$ is a \emph{weak agnostic learner} with failure probability $\beta$ with respect to this distribution if the following two conditions hold when ${\sf WAL}_{{\cal C}, \varepsilon}$ is run on input $\varepsilon$ and i.i.d. copies of $(x, y)$:
    \begin{itemize}
        \item If there exists $c \in {\cal C}$ such that $\E[c_x y] > \varepsilon$, then ${\sf WAL}_{{\cal C}}$ outputs $c' \in {\cal C}$ such that $\E[c'_x y] > \varepsilon/2$ with probability at least $1-\beta$.
        \item If every $c \in {\cal C}$ satisfies $\E[c_x y] \le \varepsilon$, then ${\sf WAL}_{{\cal C}}$ outputs $c' \in {\cal C}$ such that $\E[c'_x y] > \varepsilon/2$ or $\bot$ with probability at least $1 - \beta$.
    \end{itemize}
\end{definition}

\begin{algorithm}
    \caption{Outcome Indistinguishability via Randomized Distinguisher Selection}
    \label{alg:oi-via-random-signs}
    \DontPrintSemicolon
    \SetKw{Return}{return}
    \SetKwProg{Def}{def}{:}{end}
    \SetKwFunction{Label}{Label}
    \SetKwFunction{Select}{Select-Distinguisher}
    \SetKwFunction{Construct}{Construct}

    \vspace{0.2cm}
    
    \Def{\Label{$j, o, \tilde{p}, E$}}{
        $o' \sim \tilde{p}_j$ \tcp*{draw $o' \in {\cal O}$ from the distribution $\tilde{p}_j$}
        \Return ${\bf 1}[(1, o', \hat{p}_j) \in E] - {\bf 1}[(1, o, \hat{p}_j) \in E]$\;
    }
    
    \vspace{0.2cm}
    
    \Def{\Select{$\varepsilon, \tilde{p}, \mathsf{WAL}_{{\cal C}}$}}{
        $S \gets O(\log(1/\beta))$ \tcp*{set parameters}
        $T \gets O(({\sf VC}({\cal C}) + \log(S/\beta))/\varepsilon^2)$\;
        \For(\tcp*{one-sided error amplification}){$s = 1, \ldots, S$}{
            $E_s \sim 2^{\{1\} \times {\cal O} \times {\cal G}}$ \tcp*{draw $E_s \subseteq \{1\} \times {\cal O} \times {\cal G}$ uniformly at random}
            \For{$t = 1, \ldots, T$}{
                $(i_{st}, o^*_{i_{st}}) \sim (i, o^*_i)$ \tcp*{draw independent copy of $(i, o^*_i)$}
                $(x_{st}, y_{st}) \gets \left((i_{st}, o^*_{i_{st}}), \; \text{\Label{\(i_{st}, o^*_{i_{st}}, \tilde{p}, E_s\)}}\right)$\;
            }
            $c_s \gets {\sf WAL}_{{\cal C}}\left(\varepsilon, (x_{s1}, y_{s1}), \ldots, (x_{sT} y_{sT})\right)$ \tcp*{check for $({\cal C}, \varepsilon)$-multicalibration violation}
            \If{$c_s \neq \bot$}{
                \Return $A_{c_s, E_s}$\;
            }
        }
        \Return $\bot$\;
    }
    
    \vspace{0.2cm}
    
    \Def{\Construct{$\varepsilon, \tilde{p}, \mathsf{WAL}_{{\cal C}}, {\sf update}$}}{
        $\varepsilon' \gets \varepsilon / 8\sqrt{|{\cal O}||{\cal G}|}$\;
        $A \gets$ \Select{$\varepsilon', \tilde{p}, {\sf WAL}_{{\cal C}}$}\;
        \If{$A = \bot$}{
            \Return $\tilde{p}$\;
        }
        \For{$j \in {\cal X}$}{
            \For{$o \in {\cal O}$}{
                $L_j(o) \gets A(j, o, \tilde{p})$ \tcp*{define loss function}
            }
            $\tilde{p}'_j \gets {\sf update}(\tilde{p}_j, L_j)$ \tcp*{no-regret update}
        }
        \Return \Construct{$\varepsilon, \tilde{p}', \mathsf{WAL}_{{\cal C}}, {\sf update}$} \tcp*{recurse}
    }
    
\end{algorithm}

\begin{lemma}
    \label{thm:random-signs-lemma}
    Let ${\cal A} = {\cal A}_{{\cal C}, {\cal G}}^{\rm MC}$ and $\varepsilon' = \varepsilon / 8\sqrt{|{\cal O}||{\cal G}|}$. The procedure $\textsc{Select-Distinguisher}(\varepsilon', \tilde{p}, {\sf WAL}_{{\cal C}})$ in Algorithm \ref{alg:oi-via-random-signs} samples at most $O(({\sf VC}({\cal C}) + \log(1/\beta))\log(1/\beta)/\varepsilon'^2)$ copies of $(i, o^*_i)$ and its output satisfies the following two properties:
    \begin{itemize}
        \item If there exists $A \in {\cal A}$ with $\Delta_A > \varepsilon$, then the procedure returns $A' \in {\cal A}$ with $\Delta_{A'} > \varepsilon'/2$ with probability at least $1 - \beta$.
        \item If every $A \in {\cal A}$ satisfies $\Delta_A \le \varepsilon$, then the procedure returns either $A' \in {\cal A}$ with $\Delta_{A'} > \varepsilon' / 2$ or $\bot$ with probability at least $1 - \beta$.
    \end{itemize}
\end{lemma}

\begin{proof}
    Suppose first that there exists a distinguisher $A_{c, E} \in {\cal A}$ with advantage $\Delta_{A_{c,E}} > \varepsilon$, where $c \in {\cal C}$ and $E \subseteq \{0, 1\} \times {\cal O} \times {\cal G}$. This means that \[\Pr[(c_i, \tilde{o}_i, \hat{p}_i) \in E] - \Pr[(c_i, o^*_i, \hat{p}_i) \in E] > \varepsilon,\] which can also be written as \[\E\big[{\bf 1}[(c_i, \tilde{o}_i, \hat{p}_i) \in E] - {\bf 1}[(c_i, o^*_i, \hat{p}_i) \in E]\big] > \varepsilon.\] It follows that at least one of the two inequalities \[\E\Big[c_i \Big({\bf 1}[(1, \tilde{o}_i, \hat{p}_i) \in E] - {\bf 1}[(1, o^*_i, \hat{p}_i) \in E]\Big)\Big] > \frac{\varepsilon}{2}\] or \[\E\Big[(1 - c_i) \Big({\bf 1}[(0, \tilde{o}_i, \hat{p}_i) \in E] - {\bf 1}[(0, o^*_i, \hat{p}_i) \in E]\Big)\Big] > \frac{\varepsilon}{2}\] must hold. If ${\cal C}$ is closed under complement, we may assume without loss of generality that the first inequality holds. Using the expression $\textsc{Label}$ defined in the algorithm, we may rewrite this inequality as \[\E\left[c_i \cdot \textsc{Label}(i, o^*_i, \tilde{p}, E)\right] > \frac{\varepsilon}{2}.\]

    Consider a fixed iteration $s \in [S]$ in the main loop of procedure $\textsc{Select-Distinguisher}$. Using standard anti-concentration inequalities, we will show that \[\E\Big[c_i \cdot \textsc{Label}(i, o^*_i, \tilde{p}, E_s) \;\Big|\; E_s\Big] > \varepsilon'\] with probability at least $\Omega(1)$ (over the randomness in the choice of $E_s \subseteq \{1\} \times {\cal O} \times {\cal G}$).

    Once this is shown, it will follow that for an appropriate choice of the number of iterations $S = O(\log(1/\beta))$, with probability at least $1 - \beta/2$ over the draws of $E_1, \ldots, E_S$, there will exist some iteration $s \in [S]$ such that the set $E_s$ has the above property.

    Additionally, ${\sf WAL}_{{\cal C}}$ succeeds in all $S$ iterations with probability at least $1 - \beta/2$ since it is fed $T = O(({\sf VC}({\cal C}) + \log(S/\beta)/(\varepsilon')^2)$ fresh labeled samples in each iteration.
    
    Consequently, $\textsc{Select-Distinguisher}$ returns a distinguisher $A_{c_s, E_s}$ such that \[\Delta_{A_{c_s, E_s}} = \E\Big[c_{si} \cdot \textsc{Label}(i, o^*_i, \tilde{p}, E_s) \;\Big|\; c_s, E_s\Big] > \frac{\varepsilon}{2}\] with probability at least $1 - \beta$ over all of the internal randomness of the algorithm (i.e., the draws of $E_1, \ldots E_S$ and the samples $(i_{st}, o^*_{i_{st}})$ fed to ${\sf WAL}_{{\cal C}}$).

    In the case that every distinguisher $A_{c,E} \in {\cal A}$ satisfies $\Delta_{A_{c,E}} \le 3\sqrt{|{\cal O}||{\cal G}|}\varepsilon$, the error guarantee of ${\sf WAL}_{{\cal C}}$ similarly ensures that the procedure returns either $A' \in {\cal A}$ with $\Delta_{A'} > \varepsilon / 2$ or $\bot$ with probability at least $1-\beta$.

    It remains to prove the claim that the existence of a set $E \subseteq \{0, 1\} \times {\cal O} \times {\cal G}$ satisfying \[\E[c_i \cdot \textsc{Label}(i, o^*_i,\tilde{p}, E)] > \frac{\varepsilon}{2}\] implies that \[\E[c_i \cdot \textsc{Label}(i, o^*_i,\tilde{p}, E_s) \mid E_s] > \varepsilon'\] with probability at least $\Omega(1)$ over the draw of a uniformly random $E_s \subseteq \{1\} \times {\cal O} \times {\cal G}$. To begin, define the sign $\sigma_{E_s}(o, v)$ to be $1$ if $(1, o, v) \in E_s$ and $-1$ if $(1, o, v) \notin E_s$. It is clear that the collection of random variables $\{ \sigma_{E_s}(o, v) \mid o \in {\cal O}, v \in {\cal G}\}$ are independent Rademacher random variables.

    Some algebra shows that \[\E[c_i \textsc{Label}(i, o^*_i,\tilde{p}, E_s) \mid E_s] = \frac{1}{2}\sum_{\substack{o \in {\cal O} \\ v \in {\cal G}}} \sigma_{E_s}(o, v) \Big(\Pr[c_i = 1, \tilde{o}_i = o, \hat{p}_i = v] - \Pr[c_i = 1, o^*_i = o, \hat{p}_i = v]\Big),\] which clearly has mean $0$ and standard deviation \[\frac{1}{2}\left(\sum_{\substack{o \in {\cal O} \\ v \in {\cal G}}}\Big(\Pr[c_i = 1, \tilde{o}_i = o, \hat{p}_i = v] - \Pr[c_i = 1, o^*_i = o, \hat{p}_i = v]\Big)^2\right)^{\frac{1}{2}}.\]

    By the Cauchy-Schwarz inequality, this standard deviation is at least 
    \[\frac{1}{2\sqrt{|{\cal O}||{\cal G}|}}\sum_{\substack{o \in {\cal O} \\ v \in {\cal G}}} \Big| \Pr[c_i = 1, \tilde{o}_i = o, \hat{p}_i = v] - \Pr[c_i = 1, o^*_i = o, \hat{p}_i = v] \Big|,\] which is at least \[\frac{1}{2\sqrt{|{\cal O}||{\cal G}|}} \E[c_i \cdot \textsc{Label}(i, o^*_i, \tilde{p}, E)] > \frac{\varepsilon}{4\sqrt{|{\cal O}||{\cal G}|}} = 2\varepsilon'.\] To conclude the proof, we observe any symmetric random variable with mean $0$ and standard deviation $2\varepsilon'$ must exceed $\varepsilon'$ with probability $\Omega(1)$ (e.g., by the Paley-Zygmund inequality).
\end{proof}

\begin{theorem}
    \label{thm:multicalibration-samples-random-signs}
    Fix ${\cal C} \subseteq \{0, 1\}^{\cal X}$ and let $\ell = |{\cal O}|$. There is a constant $c > 0$ and an algorithm that takes as input \((c/\varepsilon)^{2\ell + c} \cdot {\sf VC}({\cal C})\) i.i.d. individual-outcome pairs and outputs a $({\cal C}, \varepsilon)$-multicalibrated predictor w.p. $99\%$ while only accessing ${\cal C}$ through calls to ${\sf WAL}_{{\cal C}}$.
\end{theorem}

\begin{proof}
    The proof of Theorem \ref{thm:oi-mirror-descent} with a step size proportional to $\varepsilon'$, along with the two properties in Lemma \ref{thm:random-signs-lemma}, gives a $O(\log(\ell)/\varepsilon'^2)$ upper bound on the number of updates. Lemma \ref{thm:random-signs-lemma} also upper bounds the number of samples needed per update by $O(\sf{VC}({\cal C})/\varepsilon'^2)$. Choosing $\beta$ appropriately and applying a union bound over the entire sequence of updates and substituting $\varepsilon' = \Theta(\varepsilon / \sqrt{|{\cal O}||{\cal G}|})$ yields the claimed sample complexity upper bound.
\end{proof}

\section{Graph Regularity as Structured Multicalibration}
\label{sec:graph-regularity}

\emph{Szemer\'{e}di's regularity lemma} \cite{szemeredi1975regular} is a cornerstone result in extremal graph theory with a wide range of applications in combinatorics, number theory, computational complexity theory, and other areas of mathematics. Roughly speaking, it states that any large, dense graph can be decomposed into parts that behave ``pseudorandomly'' in a certain precise sense. The \emph{Frieze-Kannan weak regularity lemma} \cite{frieze1996regularity} is a related result in graph theory with a qualitatively weaker conclusion, but parameter dependencies much better suited for algorithmic applications.

The goal of this section is to show that regularity partitions of a graph correspond to predictors satisfying multi-group fairness and an additional structural condition on their level sets. In Sections~\ref{sec:graph-regularity-definitions}, \ref{sec:graph-relationships}, and \ref{sec:intermediate-regularity-algorithm}, we state various definitions of graph regularity. In Section~\ref{sec:graph-regularity-reduction}, we state and prove thecorrespondence, which is the key result.

\subsection{Definitions of Graph Regularity}
\label{sec:graph-regularity-definitions}

Let $G = (V, E)$ be a graph, by which we mean that $V$ is a finite set and $E \subseteq V \times V$. For vertex subsets $S \subseteq V$ and $T \subseteq V$, let $e_G(S, T) = (S \times T) \cap E$ count the number of edges from $S$ to $T$, and let $d_G(S, T) = e_G(S, T) / |S||T|$ denote the \emph{density} of edges from $S$ to $T$. When the graph $G$ is clear from context, we will omit the subscript $G$ from $e_G$ and $d_G$.

To state Szemer\'{e}di's regularity lemma, we must first recall the notion of an $\varepsilon$-regular pair:

\begin{definition}
Let $X, Y \subseteq V$. We say that the pair $(X, Y)$ is \emph{$\varepsilon$-regular} if \[|d(S, T) - d(X, Y)| \le \varepsilon\] for all $S \subseteq X$ and $T \subseteq Y$ such that $|S| \ge \varepsilon|X|$ and $|T| \ge \varepsilon |Y|$.
\end{definition}

Intuitively, a pair $(X, Y)$ is $\varepsilon$-regular if edges from $X$ to $Y$ are distributed in a ``pseudorandom'' fashion. The Szemer\'{e}di regularity lemma finds a partition ${\cal P}$ of the vertices of $V$ such that most pairs of parts are $\varepsilon$-regular, in the following sense: 

\begin{definition}
\label{def:szemeredi-regularity}
A partition ${\cal P} = \{V_1, \ldots, V_m\}$ of $V$ satisfies \emph{Szemer\'{e}di $\varepsilon$-regularity} if \[\sum_{\substack{j,k \in [m] \\ (V_j, V_k)\text{ not }\varepsilon\text{-regular}}} |V_j||V_k| \le \varepsilon|V|^2.\]
\end{definition}

In contrast to Szemer\'{e}di regularity, which gives fine-grained ``local'' regularity guarantees on the pairs of regular parts, the weaker regularity condition of \cite{frieze1996regularity} gives only a coarse ``global'' regularity guarantee:

\begin{definition}
A partition ${\cal P} = \{V_1, \ldots, V_m\}$ of $V$ satisfies \emph{Frieze-Kannan $\varepsilon$-regularity} if for all $S, T \subseteq V$, \[\left|e(S, T) - \sum_{j,k \in [m]}d(V_j, V_k)|S \cap V_j||T \cap V_k|\right| \le \varepsilon|V|^2~.\] 
\end{definition}

\paragraph{Intermediate Regularity} We will soon show that for a certain instantiation of the multi-group fairness framework, Szemer\'{e}di regularity corresponds to strict multicalibration, and Frieze-Kannan regularity corresponds to multiaccuracy. Inspired by this connection, we will also show that (ordinary) multicalibration corresponds to an \emph{intermediate} notion of graph regularity that has, to our knowledge, not appeared in the prior literature:

\begin{definition}
Let $X, Y, S, T \subseteq V$. We say that the pair $(X, Y)$ is \emph{$(S, T, \varepsilon)$-regular} if \[|d(S \cap X, T \cap Y) - d(X, Y)| \le \varepsilon.\]
\end{definition}

\begin{definition}
\label{def:intermediate-regularity}
A partition ${\cal P} = \{V_1, \ldots, V_m\}$ of $V$ satisfies \emph{intermediate $\varepsilon$-regularity} if for all $S, T \subseteq V$, \[\sum_{\substack{j,k \in [m] \\ (V_j, V_k)\text{ not }(S, T, \varepsilon)\text{-regular}}} |S \cap V_j||T \cap V_k| \le \varepsilon|V|^2~.\] 
\end{definition}

We chose the name \emph{intermediate regularity} to emphasize that it is a strictly stronger notion than Frieze-Kannan weak regularity, but still strictly weaker than Szemer\'{e}di regularity. In Section~\ref{sec:graph-relationships}, we prove these claimed relationships. In Section~\ref{sec:intermediate-regularity-algorithm}, we present an algorithm for achieving intermediate regularity.

\subsection{Properties of Intermediate Regularity}
\label{sec:graph-relationships}

This section is devoted to the following two results, which establish the strict separation of our notion of intermediate regularity from Szemer\'{e}di regularity and from Frieze-Kannan regularity.

\begin{theorem}
\label{thm:frieze-kannan-intermediate-separation}
There is an absolute constant $c \in (0, 1)$ such that for all sufficiently small $\varepsilon > 0$:
\begin{itemize}
    \item For any graph $G$, if the vertex partition ${\cal P}$ satisfies intermediate $\varepsilon$-regularity, then ${\cal P}$ satisfies Frieze-Kannan $\varepsilon^c$-regularity.
    \item There exists a graph $G$ and a vertex partition ${\cal P}$ satisfying Frieze-Kannan $\varepsilon$-regularity but not intermediate $c$-regularity.
\end{itemize}
\end{theorem}

\begin{theorem}
\label{thm:intermediate-szemeredi-separation}
There is an absolute constant $c \in (0, 1)$ such that for all sufficiently small $\varepsilon > 0$:
\begin{itemize}
    \item For any graph $G$, if the vertex partition ${\cal P}$ satisfies Szemer\'{e}di $\varepsilon$-regularity, then ${\cal P}$ satisfies intermediate $\varepsilon^c$-regularity.
    \item There exists a graph $G$ such that any vertex partition ${\cal P}$ satisfying intermediate $\varepsilon$-regularity does not satisfy Szemer\'{e}di $\frac{c}{\sqrt{\log^*(1/\varepsilon)}}$-regularity.
\end{itemize}
\end{theorem}

In order to prove Theorems~\ref{thm:frieze-kannan-intermediate-separation} and \ref{thm:intermediate-szemeredi-separation}, it will be useful to introduce an alternative characterization of Szemer\'{e}di regularity, based on the notion of \emph{irregularity}:

\begin{definition}
Let $X, Y \subseteq V$. The \emph{irregularity} of the pair $(X, Y)$ is \[{\rm irreg}(X, Y) = \max_{\substack{S \subseteq X \\ T \subseteq Y}}\big|e(S, T) - d(X, Y)|S||T|\big|.\]
\end{definition}

Specifically, it is known that Szemer\'{e}di $\varepsilon$-regularity (Definition~\ref{def:szemeredi-regularity}) is equivalent to having irregularity at most $\varepsilon |V|^2$, up to a polynomial change in $\varepsilon$, where the irregularity of a partition is defined as follows:

\begin{definition}
The \emph{irregularity} of a partition ${\cal P} = \{V_1, \ldots, V_m\}$ of $V$ is \[{\rm irreg}({\cal P}) = \sum_{j,k \in [m]} {\rm irreg}(V_j, V_k). \]
\end{definition}

For more on this equivalence, we refer the reader to \cite{skorski2017cryptographic}. One can state an alternate version of the definition of intermediate regularity (Definition~\ref{def:intermediate-regularity}) that is equivalent up to a polynomial change in the $\varepsilon$ parameter. Specifically, one can check that intermediate $\varepsilon$-regularity is equivalent to having \emph{$(S, T)$-irregularity} at most $\varepsilon |V|^2$ for all $S, T \subseteq V$, up to a polynomial change in $\varepsilon$, where the $(S, T)$-irregularity of a partition is defined as follows:

\begin{definition}
Let $X, Y, S, T \subseteq V$. The \emph{$(S, T)$-irregularity} of the pair $(X, Y)$ is \[{\rm irreg}_{S, T}(X, Y) = \Big|e(S \cap X, T \cap Y) - d(X, Y)|S \cap X||T \cap Y|\Big|.\]
\end{definition}

\begin{definition}
\label{def:intermediate-irregularity}
The \emph{$(S, T)$-irregularity} of a partition ${\cal P} = \{V_1, \ldots, V_m\}$ of $V$ is \[{\rm irreg}_{S, T}({\cal P}) = \sum_{j,k \in [m]} {\rm irreg}_{S, T}(V_j, V_k). \]
\end{definition}

We now formally state and prove the claimed equivalence between Definitions~\ref{def:intermediate-regularity} and \ref{def:intermediate-irregularity}:

\begin{theorem}
\label{thm:intermediate-regularity-equivalence}
If ${\cal P}$ satisfies intermediate $\varepsilon$-regularity, then ${\cal P}$ has $(S, T)$-irregularity at most $2\varepsilon|V|^2$ for all $S, T \subseteq V$. Conversely, if ${\cal P}$ has $(S, T)$-irregularity at most $\varepsilon |V|^2$ for all $S, T \subseteq V$, then ${\cal P}$ satisfies intermediate $\sqrt{\varepsilon}$-regularity.
\end{theorem}

\begin{proof}
    To prove the forward direction, suppose that the partition  ${\cal P} = \{V_1, \ldots, V_m\}$ satisfies intermediate $\varepsilon$-regularity. If $(V_j, V_k)$ is an $(S, T, \varepsilon)$-regular pair, then \[{\rm irreg}_{S, T}(V_j, V_k) \le \varepsilon |S \cap V_j||T \cap V_k|.\] If $(V_j, V_k)$ is not an $(S, T, \varepsilon)$-regular pair, we only have the bound \[{\rm irreg}_{S, T}(V_j, V_k) \le |S \cap V_j||T \cap V_k|.\] Consequently, \[{\rm irreg}_{S, T}({\cal P}) \le \sum_{\substack{j, k \in [m] \\ (V_j, V_k) \text{ is }(S, T, \varepsilon)\text{-regular}}} \varepsilon|S \cap V_j||T \cap V_k| + \sum_{\substack{j, k \in [m] \\ (V_j, V_k) \text{ not }(S, T, \varepsilon)\text{-regular}}} |S \cap V_j||T \cap V_k|.\] The first sum is clearly at most $\varepsilon|V|^2$, and the second sum is at most $\varepsilon|V|^2$ by Definition \ref{def:intermediate-regularity}.
    
    To prove the converse direction, suppose that ${\cal P}$ has $(S, T)$-irregularity at most $\varepsilon|V|^2$ for all $S, T \subseteq V$. If a pair $(V_j, V_k)$ is not $(S, T, \sqrt{\varepsilon})$-regular, then ${\rm irreg}_{S, T}(V_j, V_k) \ge \sqrt{\varepsilon}|S \cap V_j||T \cap V_k|$. It follows that the partition ${\cal P}$ under consideration satisfies \[\varepsilon |V|^2 \ge {\rm irreg}_{S, T}({\cal P}) \ge \sum_{\substack{j,k \in [m] \\ (V_j, V_k)\text{ not }(S, T, \sqrt{\varepsilon})\text{-regular}}} \sqrt{\varepsilon}|S \cap V_j||T \cap V_k|,\] and dividing both sides by $\sqrt{\varepsilon}$ allows us to conclude that ${\cal P}$ satisfies intermediate $\sqrt{\varepsilon}$-regularity.
\end{proof}

With these alternative characterizations of Szemer\'{e}di regularity and intermediate regularity in hand, we are ready to prove Theorems \ref{thm:frieze-kannan-intermediate-separation} and \ref{thm:intermediate-szemeredi-separation}, establishing the strict separation of intermediate regularity from prior notions.

\paragraph{Proof of Theorem~\ref{thm:frieze-kannan-intermediate-separation}}
For the first part, it suffices to show that ${\cal P}$ satisfies Frieze-Kannan $\varepsilon$-regularity if it has $(S, T)$-irregularity at most $\varepsilon|V|^2$ for all $S, T \subseteq V$. To this end, fix $S, T \subseteq V$ and observe that \[\left|e(S, T) - \sum_{j, k \in [m]} d_{jk}|S \cap V_j||T \cap V_k|\right| = \left|\sum_{j, k \in [m]} \big(e(S \cap V_j, T \cap V_k) - d_{jk}|S \cap V_j||T \cap V_k|\big)\right|,\] where $d_{jk} = d(V_j, V_k)$. By the triangle inequality, the right hand side is at most \[\sum_{j, k \in [m]} \Big|e(S \cap V_j, T \cap V_k) - d(V_j, V_k)|S \cap V_j||T \cap V_k|\Big| = {\rm irreg}_{S, T}({\cal P}) \le \varepsilon|V|^2. \]

For the second part, let $G = (V, E)$ be any graph with the property that \[\left|e_{G}(S, T) - \frac{1}{2}|S||T|\right| < \varepsilon |V|^2\] for all $S, T \subseteq V$. The existence of such a \emph{quasirandom} graph of density $1/2$ follows from standard probabilistic arguments, but explicit constructions are also known \cite{zhao2022graph}. We will now modify the graph $G$ into a graph $G'$ with a partition ${\cal P}$ satisfying Frieze-Kannan $\varepsilon$-regularity but not intermediate $1/9$-regularity. To do so, let $G' = (V', E')$ be the graph with $V' = V \times [2]$ and \[E' = \left\{\big((v_1, b_1), (v_2, b_2)\big) \mid (v_1, v_2) \in E \text{ or } b_1 \neq b_2 \text{ but not both} \right\}.\] This graph $G'$ can be realized as the \emph{Xor product} \cite{alon2007codes} of $G$ with a graph consisting of a single edge. We claim that the partition ${\cal P} = \big\{ \{(v, 1), (v, 2)\} \mid v \in V \big\}$ of $V'$ has the desired properties. To check Frieze-Kannan $\varepsilon$-regularity, observe that the density of edges from any part of ${\cal P}$ to another is precisely $1/2$, so it suffices to show that \[\left|e_{G'}(S', T') - \frac{1}{2}|S'||T'|\right| \le \varepsilon |V'|^2\] for any $S', T' \subseteq V'$. To this end, for $S' \subseteq V'$ and $b \in [2]$, let $S'_b = \{ v \in V \mid (v, b) \in V'\}$. Then
\begin{align*}
    \left|e_{G'}(S', T') - \frac{1}{2}|S'||T'|\right| = &\; \left||E \cap (S'_1 \times T'_1)| - \frac{1}{2}|S'_1||T'_1|\right| \\
    &+ \left||(V^2 \setminus E) \cap (S'_1 \times T'_2)| - \frac{1}{2}|S'_1||T'_2|\right| \\
    &+ \left||(V^2 \setminus E) \cap (S'_2 \times T'_1)| - \frac{1}{2}|S'_2||T'_1|\right| \\
    &+ \left||E \cap (S'_2 \times T'_2)| - \frac{1}{2}|S'_2||T'_2|\right|,
\end{align*} and, by our initial choice of $G$, each of the four terms on the right hand is at most $\varepsilon |V|^2 = \varepsilon |V'|^2/4$. To check that intermediate $1/9$-regularity fails, let $S' = T' = \{(v, b) \in V' \mid b = 1\}$. Then \[{\rm irreg}_{S', T'}({\cal P}) = \sum_{v_1, v_2 \in V} \left| {\bf 1}[(v_1, v_2) \in E] - \frac{1}{2}\right| = \frac{1}{2}|V|^2 = \frac{1}{8}|V'|^2. \]

\paragraph{Proof of Theorem~\ref{thm:intermediate-szemeredi-separation}}
For the first part, it suffices by Theorem \ref{thm:intermediate-regularity-equivalence} to show that the $(S, T)$-irregularity of ${\cal P}$ is bounded above by its irregularity. To this end, observe that for any $V_j, V_k \in {\cal P}$, we have that \[{\rm irreg}(V_j, V_k) = \max_{S, T \subseteq V} {\rm irreg}_{S, T}(V_j, V_k).\] Therefore, for any particular $S, T \subseteq V$, we have that \[{\rm irreg}_{S, T}({\cal P}) = \sum_{j, k \in [m]} {\rm irreg}_{S, T}(V_j, V_k) \le \sum_{j, k \in [m]} {\rm irreg}(V_j, V_k) = {\rm irreg}({\cal P}). \]

The second part follows readily from the following two facts. The first fact is a lower bound on the number of parts required to achieve Szemer\'{e}di $\varepsilon$-regularity. Specifically, \cite{fox2014tight} showed that there exists a graph for which every vertex partition ${\cal P}$ with \[{\rm irreg}({\cal P}) \le \varepsilon |V|^2\] requires the number of parts $|{\cal P}|$ to be at least a tower of twos of height $\Omega(1/\varepsilon^{2})$. The second fact is an upper bound on the number of parts required to achieve intermediate $\varepsilon$-regularity. Specifically, we will argue in Section~\ref{sec:intermediate-regularity-algorithm} that every graph has a vertex partition ${\cal P}$ with \[\max_{S, T \subseteq V} {\rm irreg}_{S, T}({\cal P}) \le \varepsilon |V|^2\] and $|{\cal P}| \le 4^{1/\varepsilon^2}$. Comparing these upper and lower bounds yields the claimed separation between intermediate and Szemer\'{e}di regularity.

\subsection{Algorithm for Intermediate Regularity}
\label{sec:intermediate-regularity-algorithm}

Theorem~\ref{thm:regularity-fairness} suggests that intermediate regularity might be achievable via a modified multicalibration algorithm. This is indeed the case, and standard analyses show that Algorithm~\ref{alg:intermediate-regularity}, initialized with the trivial partition ${\cal P} = \{V\}$, computes such a partition with complexity summarized by the second row of Table~\ref{tab:regularity-bounds}. The algorithm can also be viewed as a modification of a standard algorithm for Frieze-Kannan regularity. The \textsc{Select} subroutine of Algorithm~\ref{alg:intermediate-regularity} implements the algorithm from \cite{alon2004approximating} that takes as input a function $f : V \times V \to [-1, 1]$ and outputs $S, T \subseteq V$ such that \(\left|\sum_{(u,v) \in S \times T} f(u,v)\right| > \frac{1}{2} \max_{S', T' \subseteq V} \left|\sum_{(u, v) \in S' \times T'} f(u,v)\right|\) in ${\rm poly}(|V|)$ time.

\begin{table}[htb]
    \centering
    \begin{tabular}{c|c|c}
        Regularity Notion & Number of Parts & Time Complexity\tablefootnote{The time complexity bound for Frieze-Kannan regularity is for computing an implicit representation of the partition.} \\
        \hline
        Frieze-Kannan & $\exp({\rm poly}(1/\varepsilon))$ & ${\rm poly}(n/\varepsilon)$\\
        Intermediate & $\exp({\rm poly}(1/\varepsilon))$ & ${\rm poly}(n) \exp(\exp({\rm poly}(1/\varepsilon)))$ \\
        Szemer\'{e}di & ${\rm tower}({\rm poly}(1/\varepsilon))$ & ${\rm poly}(n) {\rm tower}({\rm poly}(1/\varepsilon))$
    \end{tabular}
    \caption{Some Upper Bounds on Achieving Graph Regularity ($n = |V|$, ${\rm tower}(h) = \underbrace{2^{2^{\cdot^{\cdot^{\cdot^{2}}}}}}_{h\text{ times}}$)}
    \label{tab:regularity-bounds}
\end{table}

\begin{algorithm}[tb]
    \caption{Construction of Intermediate $\varepsilon$-Regularity Partition}
    \label{alg:intermediate-regularity}
    \SetKwProg{Def}{def}{:}{end}
    \SetKwFunction{Refine}{\textsc{\textrm{Refine}}}
    \SetKwFunction{Select}{\textsc{\textrm{Select}}}
    \Def{\Refine{$\varepsilon$, $G$, ${\cal P}$}}{
        \For{$\sigma \in \{-1, +1\}^{m \times m}$}{
            $S, T \gets$ \Select{$\sum_{j=1}^m \sum_{k=1}^m \sigma_{jk}({\bf 1}_E - d(V_j, V_k)){\bf 1}_{V_j \times V_k}$} \\
            \If{${\rm irreg}_{S, T}({\cal P}) > \frac{1}{2}\varepsilon|V|^2$}{
                ${\cal P}' \gets$ common refinement of $S$, $T$, and ${\cal P}$ \\
                \Return \Refine{$\varepsilon$, $G$, ${\cal P}'$}
            }
        }
        \Return ${\cal P}$
    }
\end{algorithm}

\subsection{The Regularity-Multicalibration Theorem}
\label{sec:graph-regularity-reduction}

\begin{figure}[tb]
    \centering
    \includegraphics{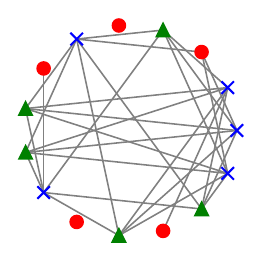}
    \includegraphics{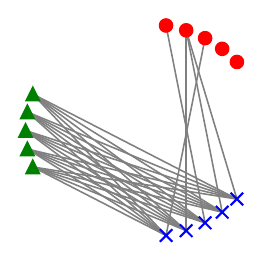}
    \includegraphics{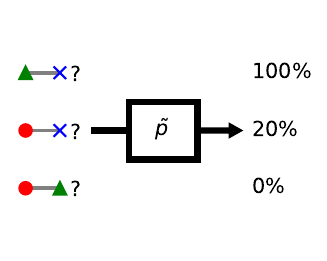}
    \caption{A graph $G$, a regularity partition of $G$, and a multicalibrated edge predictor for $G$.}
    \label{fig:graph}
\end{figure}

The Szemer\'{e}di and Frieze-Kannan regularity lemmas state that any graph $G$ has a partition satisfying $\varepsilon$-regularity (of the appropriate kind) whose number of parts is bounded by a function of $\varepsilon$. The partitions constructed in the proofs of these lemmas can be viewed as low-complexity approximations to the graph that fool a particular family of cryptographic distinguishers, as observed by \cite{skorski2017cryptographic}.

In this section, we first show that these distinguishers fit neatly within the framework of multi-group fairness. In particular, we will see that Frieze-Kannan weak regularity corresponds naturally to multiaccuracy and that Szemer\'{e}di regularity corresponds naturally to \emph{strict} multicalibration, with respect to the \emph{same} collection of subpopulations. Taking this connection one step further, we show that (ordinary) multicalibration naturally gives rise to our new notion of \emph{intermediate regularity}, which is stronger than Frieze-Kannan regularity but weaker than Szemer\'{e}di regularity.

\begin{definition}
    In the \emph{edge prediction problem} for a graph $G = (V, E)$, the population is ${\cal X} = V \times V$, each individual $i \in V \times V$ is a vertex pair drawn uniformly at random, and the true outcome of individual $i$ is the single bit $o^*_i = {\bf 1}[i \in E]$. The collection of protected subpopulations is ${\cal C} = \{{\bf 1}_{S \times T} \mid S, T \subseteq V\}$. In this setting, we call $\tilde{p} : {\cal X} \to [0, 1]$ an \emph{edge predictor for $G$}.
\end{definition}

\begin{theorem}[Regularity-Multicalibration]
\label{thm:regularity-fairness}
Given a graph $G$, consider the following definitions of \emph{fairness} for an edge predictor $\tilde{p}$ for $G$ and \emph{regularity} for a vertex partition ${\cal P}$ of $G$:
\begin{itemize}
    \item[(i)] multiaccuracy and Frieze-Kannan regularity.
    \item[(ii)] multicalibration and intermediate regularity.
    \item[(iii)] strict multicalibration and Szemer\'{e}di regularity.
\end{itemize}

For each such pair of definitions, there exists an absolute constant $0 < c < 1$ such that the following two implications hold for sufficiently small $\varepsilon$:

\begin{itemize}
    \item[(a)] If ${\cal P}$ is $\varepsilon$-regular, the predictor $\tilde{p}$ that outputs $d(V_j, V_k)$ on all of $V_j \times V_k \in {\cal P}^2$ is $({\cal C}, \varepsilon^c)$-fair.\footnote{For a partition ${\cal P} = \{V_1, \ldots, V_m\}$ of the vertices $V$, the set ${\cal P}^2 = \{V_j \times V_k \mid j,k \in [m]\}$ denotes the partition of $V \times V$ obtained from all pairwise Cartesian products of parts of ${\cal P}$.}
    \item[(b)] If $\tilde{p}$ is $({\cal C}, \varepsilon)$-fair and the set of level sets of $\tilde{p}$ is ${\cal P}^2$ for some partition ${\cal P}$, then ${\cal P}$ is $\varepsilon^c$-regular.
\end{itemize}

\end{theorem}

Before presenting the formal proof, we provide a short proof sketch to emphasize the main observations. For $V_j, V_k \in {\cal P}$, define $d_{jk} = d(V_j, V_k)$. In Section~\ref{sec:graph-relationships}, we showed through algebraic manipulations that the regularity criteria of Section~\ref{sec:graph-regularity-definitions} are equivalent, up to a polynomial change in $\varepsilon$, to the following conditions:
\begin{align*}
    \max_{S, T \subseteq V} \left| \sum_{j=1}^m \sum_{k=1}^m e(S \cap V_j, T \cap V_k) - d_{jk}|S \cap V_j||T \cap V_k|\right| &\le \varepsilon |V|^2 &\text{for Frieze-Kannan $\varepsilon$-regularity,} \\
    \max_{S, T \subseteq V} \sum_{j=1}^m \sum_{k=1}^m\left|e(S \cap V_j, T \cap V_k) - d_{jk}|S \cap V_j||T \cap V_k|\right| &\le \varepsilon |V|^2 &\text{for intermediate $\varepsilon$-regularity,} \\
    \sum_{j=1}^m \sum_{k=1}^m\max_{S, T \subseteq V} \left|e(S \cap V_j, T \cap V_k) - d_{jk}|S \cap V_j||T \cap V_k|\right| &\le \varepsilon |V|^2 &\text{for Szemer\'{e}di $\varepsilon$-regularity.}
\end{align*}
The relationships among our fairness criteria can be phrased similarly. Indeed, let \(\Delta_{S,v}(\tilde{p}) = \Pr[i \in S, \tilde{o}_i = 1, \tilde{p}_i = v] - \Pr[i \in S, o^*_i = 1, \tilde{p}_i = v]\) for a predictor $\tilde{p} : {\cal X} \to [0, 1]$, a subpopulation $S \subseteq {\cal X}$ and a value $v \in [0, 1]$. Then, the requirements of Definitions~\ref{def:multiaccuracy}, \ref{def:multicalibration}, and \ref{def:strict-multicalibration} reduce to:
\begin{align*}
    \max_S \left|\sum_v \Delta_{S,v}(\tilde{p})\right| &\le \varepsilon &\text{for $({\cal C}, \varepsilon)$-multiaccuracy,}\\
    \max_S \sum_v \left|\Delta_{S,v}(\tilde{p})\right| &\le \varepsilon &\text{for $({\cal C}, \varepsilon)$-multicalibration,}\\
    \sum_v \max_S \left|\Delta_{S,v}(\tilde{p})\right| &\le \varepsilon &\text{for strict $({\cal C}, \varepsilon)$-multicalibration,}
\end{align*}
where the maxima are taken over $S$ such that ${\bf 1}_S \in {\cal C}$ or ${\bf 1}_{{\cal X} \setminus S} \in {\cal C}$ or $S = {\cal X}$ and the sums are taken over the range of $\tilde{p}$.
Comparing the two displayed sets of inequalities will yield part (a) of the theorem. Part (b) will follow from a similar argument.

\begin{proof}[Proof of Theorem~\ref{thm:regularity-fairness}]
Throughout this proof, we will use the notation $d_{jk}$ and $\Delta_{S,v}$ from the proof sketch, as well as the alternate characterizations of Szemer\'{e}di and intermediate regularity from Section~\ref{sec:graph-relationships}.
\begin{itemize}
    \item[(a)] As usual, let $V_1, \ldots, V_m$ denote the parts of ${\cal P}$, and consider any fixed sets $S, T \subseteq V$ and value $v \in [0, 1]$. The construction of $\tilde{p}$ ensures that $\tilde{p}_{(a,b)} = d(V_j, V_k)$ for any vertex pair $(a,b) \in V_j \times V_k$, so some algebra yields \[\Delta_{S \times T, v}(\tilde{p}) = \sum_{\substack{(j,k) \in [m]^2 \text{ s.t. }\\\tilde{p}(V_j \times V_k) = v}} e(S \cap V_j, T \cap V_k) - d(V_j, V_k)|S \cap V_j||T \cap V_k|,\] along with the useful fact that $\Delta_{{\cal X} \setminus (S \times T), v}(\tilde{p}) = -\Delta_{S \times T, v}(\tilde{p})$. By taking absolute values, summing over $v$ in the range of $\tilde{p}$, and taking the max over $S, T \subseteq V$ in various orders on both sides of the above equation, we deduce the following three inequalities:
    \begin{align*}
        \max_{S, T \subseteq V} \left|\sum_{v \in \tilde{p}({\cal X})} \Delta_{S \times T, v}(\tilde{p})\right| &\le \max_{S, T \subseteq V} \left| \sum_{j=1}^m \sum_{k=1}^m e(S \cap V_j, T \cap V_k) - d_{jk}|S \cap V_j||T \cap V_k|\right|, \\
        \max_{S, T \subseteq V} \sum_{v \in \tilde{p}({\cal X})} \left|\Delta_{S \times T, v}(\tilde{p})\right| &\le \max_{S, T \subseteq V} \sum_{j=1}^m \sum_{k=1}^m\left|e(S \cap V_j, T \cap V_k) - d_{jk}|S \cap V_j||T \cap V_k|\right|, \\
        \sum_{v \in \tilde{p}({\cal X})} \max_{S, T \subseteq V} \left|\Delta_{S \times T, v}(\tilde{p})\right| &\le \sum_{j=1}^m \sum_{k=1}^m\max_{S, T \subseteq V} \left|e(S \cap V_j, T \cap V_k) - d_{jk}|S \cap V_j||T \cap V_k|\right|.
    \end{align*}
    The above three inequalities, show, respectively, that for an appropriate absolute constant $c \in (0, 1)$ and sufficiently small $\varepsilon > 0$, Frieze-Kannan $\varepsilon$-regularity of ${\cal P}$ implies $\varepsilon^c$-multiaccuracy of $\tilde{p}$, that intermediate $\varepsilon$-regularity of ${\cal P}$ implies $\varepsilon^c$-multicalibration of $\tilde{p}$, and that Szemer\'{e}di $\varepsilon$-regularity of ${\cal P}$ implies strict $\varepsilon^c$-multicalibration of $\tilde{p}$.
    \item[(b)] For a fixed set $S \subseteq V$ of vertices, let $\tilde{S} \subseteq V$ denote a random set of vertices sampled as follows: for each $j \in [m]$, independently include all vertices of $V_j$ in $\tilde{S}$ with probability $|S \cap V_j|/|V_j|$. 
    With this notation, a simple algebraic calculation shows that for any fixed $j,k \in [m]$ such that $V_j \times V_k$ exactly coincides with the $v$-level set of $\tilde{p}$, and for any fixed sets $S, T \subseteq V$, we have
    \begin{align*}
        e(S \cap V_j, T \cap V_k) - d(V_j, V_k)|S \cap V_j| |T \cap V_k| = \Delta_{S \times T, v}(\tilde{p}) - \E_{\tilde{S}, \tilde{T}}\left[\Delta_{\tilde{S} \times \tilde{T}, v}(\tilde{p})\right].
    \end{align*}
    We will manipulate this key equation in three different ways to derive the three versions of this part of the theorem. First, by summing over $j,k \in [m]$, taking absolute values, and applying the triangle inequality, we see that
    \begin{align*}
        \left|e(S, T) - \sum_{j,k \in [m]} d_{jk}|S \cap V_j| |T \cap V_k|\right| \le \left|\Delta_{S \times T}(\tilde{p})\right| + \E_{\tilde{S}, \tilde{T}}\left|\sum_{v \in \tilde{p}({\cal X})} \Delta_{\tilde{S} \times \tilde{T}, v}(\tilde{p})\right|.
    \end{align*}
    For small enough $\varepsilon, c > 0$, this shows that $\varepsilon$-multiaccuracy of $\tilde{p}$ implies Frieze-Kannan $\varepsilon^c$-regularity of ${\cal P}$. If we were to instead take absolute values of both sides before summing over $j, k \in [m]$ and applying the triangle inequality, we would see that
    \begin{align*}
        &\sum_{j,k \in [m]}\Big|e(S \cap V_j, T \cap V_k) - d_{jk}|S \cap V_j| |T \cap V_k|\Big|  \le \sum_{v \in \tilde{p}({\cal X})}\left|\Delta_{S \times T,v}(\tilde{p})\right| + \sum_{v \in \tilde{p}({\cal X})}\left|\Delta_{V \times V,v}(\tilde{p})\right|,
    \end{align*} which shows that $\varepsilon$-multicalibration of $\tilde{p}$ implies intermediate $\varepsilon^c$-regularity of ${\cal P}$. Finally, if we had chosen to take the maximum over $S,T$ before summing over $j,k \in [m]$ and applying the triangle inequality, we would have seen that
    \begin{align*}
        \sum_{j,k \in [m]} {\rm irreg}(V_j, V_k) \le \sum_{v \in \tilde{p}({\cal X})}\max_{S,T \subseteq V} \left|\Delta_{S \times T,v}(\tilde{p})\right| + \sum_{v \in \tilde{p}({\cal X})} \left|\Delta_{V \times V,v}(\tilde{p})\right|,
    \end{align*}
    which shows that strict $\varepsilon$-multicalibration of $\tilde{p}$ implies Szemer\'{e}di $\varepsilon^c$-regularity of ${\cal P}$.
\end{itemize}
\end{proof}
\section{Hardcore Lemma for Real-Valued Functions}
\label{sec:hardcore}

The leakage simulation lemma is connected with the hard-core lemma for deterministic Boolean functions as shown by~\cite{trevisan2009regularity, vadhan2013uniform}. We show that multicalibration enables stronger consequences, namely a hard-core lemma for real valued functions.

Informally, Impagliazzo's hard-core lemma says that for any boolean function $f : {\cal X} \to \{0, 1\}$ that is hard on average against a class $\mathcal{C} \subseteq \{0, 1\}^{\cal X}$, there is a large subset of ${\cal X}$, the {\em hard core}, whose size depends on the hardness of $f$, on which $f$ is effectively pseudorandom. We prove an analogous statement for real valued functions (equivalently, random functions with boolean outcomes) $p^*:\mathcal{X}\rightarrow[0,1]$, where hardness of being correct is replaced by hardness of approximation (in $L_2$ distance), while the hardcore/pseudorandom condition is replaced with a covariance condition. In other words, if it's hard to approximate the function, then there's some large set on which it's hard to even have any non-negligible covariance.

In order to formally state the theorem, let $R_{T}(\mathcal{C})$ denote the class of functions of relative complexity $T$ with respect to ${\cal C}$.  Let a $\mathcal{C}$ gate denote black-box computation of a function $c \in \mathcal{C}$. More precisely, every $\breve p \in R_{T}(\mathcal{C})$ is a circuit containing $T$ $\mathcal{C}$-gates and $\mathrm{poly}(T)$ additional 
basic operations (e.g., floating-point arithmetic and Boolean logical operations). In particular, this captures the class of all functions which can be the outcome of the multicalibration algorithm after $T$ rounds with an appropriate $\varepsilon$.

\begin{theorem}[Hardcore lemma for probabilities]\label{thm:hardcore}
Let $\alpha>0, \gamma \in (0,\frac{1}{3}]$. Let $\mathcal{C}\in[0,1]^{\mathcal{X}}$ be a collection of real valued functions. Let $p^*:\mathcal{X}\rightarrow[0,1]$ be a function which is hard to $\alpha$-approximate by functions in $R_{T}(\mathcal{C})$ for $T=O({\alpha^{-6}\gamma^{-4}})$, i.e.,  $\forall \breve{p}\in R_T(\mathcal{C}), \E[(p^*_i-\breve{p}_i)^2]>\alpha$.

Then, there exists a hardcore set $S\in\mathcal{X}$ where $\Pr[i\in S] > \frac{\alpha^2\gamma}{4}$ and $\forall c\in\mathcal{C},\Cov(c_i,p^*_i \mid i \in S)<\gamma\Var(p^*_i \mid i \in S)$.
\end{theorem}

The proof of this statement is based on the following intuition: First, the multicalibration algorithm gives us a relatively simple predictor $\breve{p}$. Next, a multicalibrated predictor partitions $\mathcal{X}$ into slices on which either (1) $p^*$ has low variance (so $\breve{p}$ is highly accurate), or (2) $\breve{p}$ is not highly accurate (so $p^*$ is high variance) but $\mathcal{C}$ is not able to take advantage of this (nothing in $\mathcal{C}$ is correlated with $p^*$ on this slice). Finally, if $\breve{p}$ is far from $p^*$, then there must exist a set on which the latter condition is true. That set is the hardcore.

We first prove Lemma~\ref{lem:hardcore-OI}, which states that we can obtain a relatively simple predictor satisfying both (a variant of) statistical-distance multicalibration and the additional guarantee that the predictor is perfectly accurate in expectation on all its slices.  This is obtained by a careful analysis of a simple post-processing of the predictor obtained by our multicalibration-through-outcome-indistinguishability  Algorithm~\ref{alg:outcome-indistinguishability} when instantiated with a suitable class of distinguishers.

The change is that instead of requiring indistinguishability w.r.t. the real valued $c$'s (which would require close to perfect accuracy), we ask for indistinguishability with respect to \textit{outcomes} when treating $c_i$ as a probability. That is, instead of bounding
$\delta((c_i,\breve{o}_i,\hat{p}_i),(c_i,o^*_i,\hat{p}_i))$, we bound $\delta((o^c_i,\breve{o}_i,\hat{p}_i),(o^c_i,o^*_i,\hat{p}_i))<\varepsilon$, where $o^c_i\sim\Ber(c_i)$.

We'll also need a new family of distinguishers.
\begin{definition}\label{def:oi-family-for-random-multicalibration}
Let \({\cal B}_{{\cal C}, {\cal G}}^{\mathrm{MC}} = \left\{ B_{c,E} \mid c \in {\cal C}, \; E \subseteq {\cal Y} \times {\cal O} \times {\cal G}\right\},\) for $\mathcal{C}\in\Delta\mathcal{Y}^\mathcal{X}$, where  \[B_{c,E}(j, o, \tilde{p}) = {\bf 1}[(o^c_j\leftarrow c_j, o, \hat{p}_j) \in E]\] for each member $j \in {\cal X}$, possible outcome $o \in {\cal O}$, and predictor $\tilde{p} : {\cal X} \to \Delta{\cal O}$.
\end{definition}

Note that Definition~\ref{def:oi-family-for-multicalibration} is a special case of Definition~\ref{def:oi-family-for-random-multicalibration} for trivial distributions.

\begin{lemma}
\label{lem:hardcore-OI}
    For every $\varepsilon > 0,\eta>0$, there is a function $\breve{p}\in R_{T(\varepsilon)}(\mathcal{C})$, $\breve{p}:\mathcal{X}\rightarrow[0,1]$ and a partitioning function $\hat{p} : {\cal X} \to \mathcal{G}$, with $\mathcal{G}$ an $\eta$-cover of $[0,1]$ s.t.
    \begin{enumerate}
        \item for all $c\in\mathcal{C}$, $\delta((o^c_i,\breve{o}_i,\hat{p}_i),(o^c_i,o^*_i,\hat{p}_i))<\varepsilon$.
        \item $\breve{p}$ is perfectly accurate in expectation on the slices $\hat{p}=v$.
        \item $T(\varepsilon)=O\left(\frac{1}{\varepsilon^2}\right)$
    \end{enumerate}
\end{lemma}

\begin{proof} 
Run Algorithm~\ref{alg:outcome-indistinguishability} using the collection of distinguishers ${\cal B}_{\cal C, G}^{\mathrm{MC}}$, where $\mathcal{C}$ is as provided in the statement of the theorem and $\mathcal{G}$ is the standard $\eta$ covering, to obtain $\tilde{p}$ that is
$\left({\cal B}_{\cal C, G}^{\mathrm{MC}},\frac{\varepsilon}{2}\right)$-outcome-indistinguishable from $p^*$  and has relative complexity $T(\varepsilon)=O(\varepsilon^{-2})$ to $\mathcal{C}$. This gives us a predictor satisfying
$$\delta((o^c_i,\breve{o}_i,\hat{p}_i),(o^c_i,o^*_i,\hat{p}_i))<\frac{\varepsilon}{2}$$

However, it does not necessarily satisfy Claim 2; we will modify $\tilde{p}$ to obtain a new predictor $\breve{p}$ that satisfies both Claims 1 and~2.  This is achieved by shifting $\tilde{p}$ on the level sets of $\hat{p}$, incurring an additive term of $|\mathcal{G}|$ on the complexity of $\hat{p}$. Speaking intuitively, Claim~1 still holds for this new $\hat{p}$ since we are only improving the accuracy of $\hat{p}$; we prove this intuition to be (nearly) correct. More formally, define $$\breve{p}_i:=\tilde{p}_i+\tau(\hat{p}_i),$$ where $$\tau_v:=\E[p^*_i-\tilde{p}_i|\hat{p}_i=v].$$ 
$\breve{p}_i$ satisfies Claim 1 by construction. Finally, it remains to show that Claim 1 still holds for this new $\breve{p}$. First, we show that the average magnitude of these shifts must be small:
\begin{align*}
    \E_v[|\tau_v|]
    &=\E_v[|\E[p^*_i-\tilde{p}_i|\hat{p}_i=v]|]\\
    &\le\delta((o^c_i,\tilde{o}_i,\hat{p}_i),(o^c_i,o^*_i,\hat{p}_i))
\end{align*}
We show that this implies that shifting at most doubles the statistical difference.

\begin{align*}
    \delta((o^c_i,\breve{o}_i,\hat{p}_i),(o^c_i,o^*_i,\hat{p}_i))
    &=\E_y[\delta((\breve{o}_i,\hat{p}_i),(o^*_i,\hat{p}_i)|o^c_i=y)]\\
    &=\E_y\left[\E_v\left[\left|\E_i[p^*_i-\breve{p}_i|\hat{p}_i=v]\right||o^c_i=y\right]\right]\\
    &=\E_y\left[\E_v\left[\left|\E_i[p^*_i-\tilde{p}_i-\tau_v|\hat{p}_i=v]\right||o^c_i=y\right]\right]\\
    &\le\E_y\left[\E_v\left[\left(\left|\E_i[p^*_i-\tilde{p}_i|\hat{p}_i=v]\right|+|\tau_v|\right)|o^c_i=y\right]\right]\\
    \intertext{Applying the above fact,}
    &\le\E_y\left[\E_v\left[\left|\E_i[p^*_i-\tilde{p}_i|\hat{p}_i=v]\right||o^c_i=y\right]\right]+\delta((o^c_i,\tilde{o}_i,\hat{p}_i),(o^c_i,o^*_i,\hat{p}_i))\\
    &=2\delta((o^c_i,\tilde{o}_i,\hat{p}_i),(o^c_i,o^*_i,\hat{p}_i))\\
    &\le\varepsilon.
\end{align*}

\end{proof}

We will continue to use $\mathcal{G}$ to denote the set of rounded to values/slices for a predictor. In this section, we will always use $\mathcal{G}$ of the form $(0,\eta,\dots,\lceil\frac{1}{\eta}-1\rceil\cdot\eta,1)$, so that $m=|\mathcal{G}|=\lceil\frac{1}{\eta}\rceil+1$. For clarity, we will assume from now on that $\eta = 1/m$ to avoid having to clutter the presentation with rounding.

The next lemma lets us find large slices on which the target function has large variance, under the assumption that $\tilde{p}$ fails to $\alpha$-approximate $p^*$.

\begin{lemma}\label{lem:alpha-hard}
Let $\breve{p}:[0,1]^\mathcal{X}$ and $\hat{p}:\mathcal{G}^\mathcal{X}$. If $\E[(\breve{p}_i-p_i^*)^2]>\alpha$, then there exists $v\in\mathcal{G}$ s.t. 
\begin{enumerate}
    \item $\Pr[\hat{p}_i=v]\ge 2\alpha\eta$
    \item and $\Var(p_i^*|\hat{p}_i=v)\ge \frac{\alpha}{2}-3\eta$
\end{enumerate}
\end{lemma}
\begin{proof}
First, we show that we can approximately break down $\E[(\breve{p}_i-p_i^*)^2]>\alpha$ as the average of the variance of $p^*$ on the slices defined by $\hat{p}$.
\\
\begin{align*}
\alpha
&< \E[(\breve{p}_i-p_i^*)^2]\\
&=\sum_{v\in\mathcal{G}}\Pr[\hat{p}_i=v]\E[(\breve{p}_i-p_i^*)^2|\hat{p}_i=v]\\
\intertext{Letting $\mu_v$ denote the expected value of $p^*_i$ on slice $v$ which, by construction, is equivalent to that of $\breve{p}_i$,}
&=\sum_{v\in\mathcal{G}}\Pr[\hat{p}_i=v]\E[(\breve{p}_i-\mu_v+\mu_v-p_i^*)^2|\hat{p}_i=v]\\
&=\sum_{v\in\mathcal{G}}\Pr[\hat{p}_i=v]\E[(p^*_i-\mu_v)^2 - 2(p^*_i-\mu_v)(\breve{p}_i-\mu_v) + (\breve{p}_i-\mu_v)^2\\
&=\sum_{v\in\mathcal{G}}\Pr[\hat{p}_i=v]\left(\Var(p^*_i|\hat{p}_i=v) - 2\Cov(p^*,\breve{p}|\hat{p}_i=v) + \Var(\breve{p}_i|\hat{p}_i=v)\right)\\
\intertext{Since $\breve{p}_i$ is bounded within an interval of size $\eta$,}
&\le\sum_{v\in\mathcal{G}}\Pr[\hat{p}_i=v](\Var(p^*_i|\hat{p}_i=v)+3\eta).
\end{align*}
Next, we apply a Markov type argument to show that the total probability of all the slices with large $p^*$ variance must be at least $\alpha/2$. Let $\mathcal{B}$ denote the slices on which $\Var(p^*_i|\hat{p}_i=v)\ge \frac{\alpha}{2}-3\eta$.
\begin{align*}
    &\sum_{v\in\mathcal{G}}\Pr[\hat{p}_i=v]\Var(p^*_i|\hat{p}_i=v) \\
    =&\,\sum_{v\in\mathcal{B}}\Pr[\hat{p}_i=v]\Var(p^*_i|\hat{p}_i=v) + \sum_{v\in\mathcal{G}-\mathcal{B}}\Pr[\hat{p}_i=v]\Var(p^*_i|\hat{p}_i=v) \\
    \intertext{Since $p^*$ is bounded in $[0,1]$, it has variance at most $1/4$,}
    \le&\,\sum_{v\in\mathcal{B}}\Pr[\hat{p}_i=v]\cdot\frac{1}{4} + \sum_{v\in\mathcal{G}-\mathcal{B}}\Pr[\hat{p}_i=v]\left(\frac{\alpha}{2}-3\eta\right)\\
    =&\,\frac{1}{4}\Pr[\mathcal{B}] + \frac{\alpha}{2}-3\eta.
\end{align*}
Combining with the first inequality, we conclude that
$$\Pr[\mathcal{B}]\ge 2\alpha$$
And because $\mathcal{G}$ is the uniform $\eta$ covering, we conclude that there exists a $v\in\mathcal{B}$ satisfying $\Pr[\hat{p}_i=v]\ge 2\alpha\eta$.
\end{proof}

The final two lemmas let us show that if there is no hardcore, then statistical distance must be large.
\begin{lemma}\label{lem:sd-cov}
    $\delta((o^c_i,\breve{o}_i),(o^c_i,o^*_i)\mid\hat{p}_i)\ge \Cov(c_i,p^*_i\mid\hat{p}_i)-\eta$.
\end{lemma}

\begin{proof}
We first show that $\Cov(c_i,p^*_i\mid\hat{p}_i) = \E[o^c_i(o^*_i-\hat{o}_i) | \hat{p}_i]$.

\begin{equation*}
\begin{split}
    \Cov(c_i,p^*_i\mid\hat{p}_i) 
    &= \E[c_i(p^*_i-\hat{p}_i) | \hat{p}_i]\\
    &= \E[o^c_i(o^*_i-\hat{o}_i) | \hat{p}_i]\\
\end{split}
\end{equation*}
Next, we show that $\E[o^c_i(o^*_i-\hat{o}_i) | \hat{p}_i] \leq \delta((o^c_i,\hat{o}_i,\hat{p}_i),(o^c_i,o^*_i,\hat{p}_i)\mid\hat{p}_i)$.
\begin{equation*}
\begin{split}
\E[o^c_i(o^*_i-\hat{o}_i) | \hat{p}_i]
&= \E[(o^c_i-\frac{1}{2})(o^*_i-\hat{o}_i) | \hat{p}_i]\\
&= \E[\frac{1}{2}(\Pr[o^c_i=1](\Pr[o^*_i=1]-\Pr[\tilde{o}_i=1])+\Pr[o^c_i=0](\Pr[o^*_i=0]-\Pr[\tilde{o}_i=0])]\\
&= \frac{1}{2}(\Pr[o^c_i=o^*_i]-\Pr[o^c_i-\tilde{o}_i])\\
&\leq\delta((o^c_i,\hat{o}_i),(o^c_i,o^*_i)\mid\hat{p}_i)\\
\end{split}
\end{equation*}
Putting the two together, we get
$$\delta((o^c_i,\hat{o}_i),(o^c_i,o^*_i)\mid\hat{p}_i) \geq \Cov(c_i,p^*_i\mid\hat{p}_i)$$
To conclude the proof, observe that \(\delta \big( (c_i, o^*_i), (c_i, \hat{o}_i) \mid \hat{p}_i \big) \le \delta \big( (c_i, o^*_i), (c_i, \breve{o}_i) \mid \hat{p}_i \big) + \eta\).
\end{proof}

\begin{lemma}\label{lem:hardcore-sd}
    If there exists $v^*\in\mathcal{G}$ satisfying:
    \begin{enumerate}
        \item $\Pr[\hat{p}_i=v^*]>\beta$
        \item $\Var(p^*_i|\hat{p}_i=v^*)>\lambda$
        \item $\Cov(c,p^*|\hat{p}_i=v^*)>\gamma\Var(p^*_i|\hat{p}_i=v^*)$
    \end{enumerate}
    Then
    $$\delta((c_i,\breve{o}_i,\hat{p}_i),(c_i,o^*_i,\hat{p}_i))>\beta\gamma\lambda-\beta\eta$$
\end{lemma}
\begin{proof}
    This lemma is a simple consequence of the Lemma~\ref{lem:sd-cov} and the definition of conditional statistical distance.
    \begin{align*}
        \delta((c_i,\breve{o}_i,\hat{p}_i),(c_i,o^*_i,\hat{p}_i))
        &= \sum_v\Pr[\hat{p}_i=v]\delta((c_i,\breve{o}_i,\hat{p}_i),(c_i,o^*_i,\hat{p}_i)|\hat{p}_i=v)\\
        &\ge \sum_v\Pr[\hat{p}_i=v](\Cov(c,p^*|\hat{p}_i=v)-\eta)\\
        &\ge \Pr[\hat{p}_i=v^*](\Cov(c,p^*|\hat{p}_i=v^*)-\eta)\\
        &> \beta(\gamma\Var(p^*_i|\hat{p}_i=v^*)-\eta) \\
        &> \beta\gamma\lambda-\beta\eta.
    \end{align*}
\end{proof}

Theorem~\ref{thm:hardcore} now follows by setting parameters and gluing these lemmas together.

\begin{proof}[Proof of Theorem~\ref{thm:hardcore}]
First, we apply Lemma~\ref{lem:hardcore-OI} with parameters $\varepsilon=\frac{\alpha^3\gamma^2}{16}$,$\eta=\frac{\alpha\gamma}{8}$ to attain a $\breve{p}\in R_{T(\varepsilon)}(\mathcal{C})$. $T(\varepsilon)=O\left(\frac{1}{\varepsilon^2}\right)=O\left(\frac{1}{\alpha^6\gamma^4}\right)$. Also, by the $\alpha$-hardness assumption, we can apply Lemma~\ref{lem:alpha-hard} to attain a $v^*$ satisfying $\Pr[\hat{p}_i=v^*]\ge 2\alpha\eta$ and $\Var(p_i^*|\hat{p}_i=v^*)\ge \frac{\alpha}{2}-3\eta$. Finally, by Lemma~\ref{lem:hardcore-sd}, we conclude that \textit{if there is no hard-core set} (i.e. the conditions of the lemma are satisfied with $\beta=2\alpha\eta$, $\lambda=\frac{\alpha}{2}-3\eta$, and $\gamma$ as provided), then 
$$\delta((c_i,\breve{o}_i,\hat{p}_i),(c_i,o^*_i,\hat{p}_i))>\beta\gamma\lambda-\beta\eta>\frac{\alpha^3\gamma^2}{16}=\varepsilon.$$ This violates the conclusion from our application of Lemma~\ref{lem:hardcore-OI}, that $$\delta((c_i,\breve{o}_i,\hat{p}_i),(c_i,o^*_i,\hat{p}_i))<\varepsilon.$$ We therefore conclude that there must exist a hardcore set as specified in the theorem statement.
\end{proof}

\bibliographystyle{alpha}
\bibliography{sources}

\end{document}